\pdfoutput=1
\documentclass[11pt,letterpaper,logo,onecolumn,copyright]{template_from_google}

\usepackage{mathtools}
\usepackage{multirow}
\usepackage{adjustbox}
\usepackage{nicefrac}
\usepackage{algorithm}
\usepackage{algorithmic}
\usepackage{subcaption}
\usepackage{wrapfig}
\usepackage[authoryear,sort&compress,round]{natbib}
\usepackage[capitalize,noabbrev]{cleveref}

\newcommand{\xmark}{\ding{55}}
\newcommand{\cmark}{\ding{51}}

\theoremstyle{plain}
\newtheorem{theorem}{Theorem}[section]

\theoremstyle{definition}

\theoremstyle{remark}
\newtheorem{remark}[theorem]{Remark}

\definecolor{darkblue}{rgb}{0.0,0.0,0.6}
\definecolor{darkred}{rgb}{0.7,0.0,0.0}
\hypersetup{
  pdffitwindow=true,
  pdfstartview={FitH},
  pdfnewwindow=true,
  colorlinks=true,
  linktocpage=true,
  linkcolor=darkred,
  urlcolor=darkblue,
  citecolor=darkblue,
  pdftitle={Ultra Fast PDE Solving via Physics Guided Few-step Diffusion},
  pdfauthor={Cindy Xiangrui Kong, Yueqi Wang, Haoyang Zheng, Weijian Luo, Guang Lin}
}

\makeatletter
\renewcommand\paragraph{\@startsection{paragraph}{4}{\z@}%
  {-2.5ex\@plus -1ex \@minus -.25ex}%
  {1.25ex \@plus .25ex}%
  {\itshape\normalsize\bfseries}}
\makeatother

\let\cite\citep

\title{Ultra Fast PDE Solving via Physics Guided Few-step Diffusion}
\newcommand{\shorttitle}{Phys-Instruct}
\reportnumber{}
\renewcommand{\today}{}

\author[1,*]{Cindy Xiangrui Kong}
\author[1,*]{Yueqi Wang}
\author[1]{Haoyang Zheng}
\author[2]{Weijian Luo}
\author[1]{Guang Lin}
\affil[1]{Purdue University}
\affil[2]{Independent Researcher}
\affil[*]{Equal contribution.}

\begin{abstract}
Diffusion-based models have demonstrated impressive accuracy and generalization in solving partial differential equations (PDEs). However, their practical utility still suffers from high sampling costs and insufficient physical consistency, mainly due to many-step iterative sampling and insufficient physics guidance.
To address these issues, we propose \emph{Phys-Instruct}, a physics-guided distillation framework that compresses a pre-trained diffusion PDE solver into a few-step generator while explicitly enforcing PDE consistency through distillation guidance.
We reformulate the training objective as a constrained problem with PDE residual, and derive a score-function identity that leads to a practical physics-guided generator update.
Across five PDE benchmarks, Phys-Instruct achieves orders-of-magnitude faster inference while reducing the PDE Error by factors of $2.4$ to $15.5$, compared with the corresponding 200-step EDM teachers.
Moreover, the resulting unconditional student model functions as a compact prior, enabling efficient and physically consistent inference for various downstream conditional tasks.
Our results indicate that Phys-Instruct is an effective and efficient framework for ultra-fast PDE solving powered by generative models.
\end{abstract}

\begin{document}
\maketitle

\section{Introduction}

\begin{wrapfigure}{r}{0.46\textwidth}
    \centering
    \vspace{-6pt}
    \begin{subfigure}{0.95\linewidth}
        \centering
        \includegraphics[width=\linewidth]{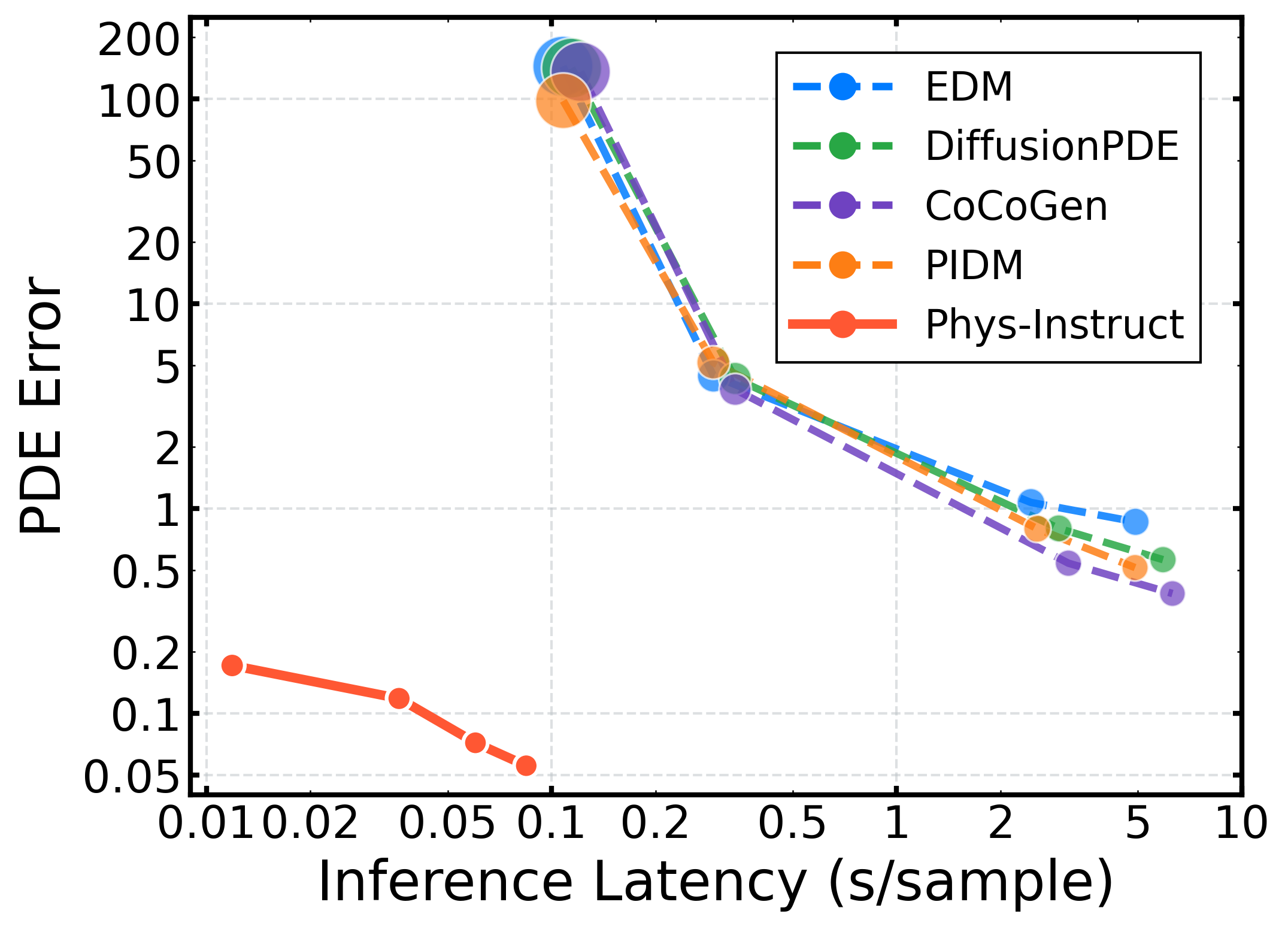}
    \end{subfigure}
    \par\vspace{2pt}
    \begin{subfigure}{0.90\linewidth}
        \centering
        \includegraphics[width=\linewidth]{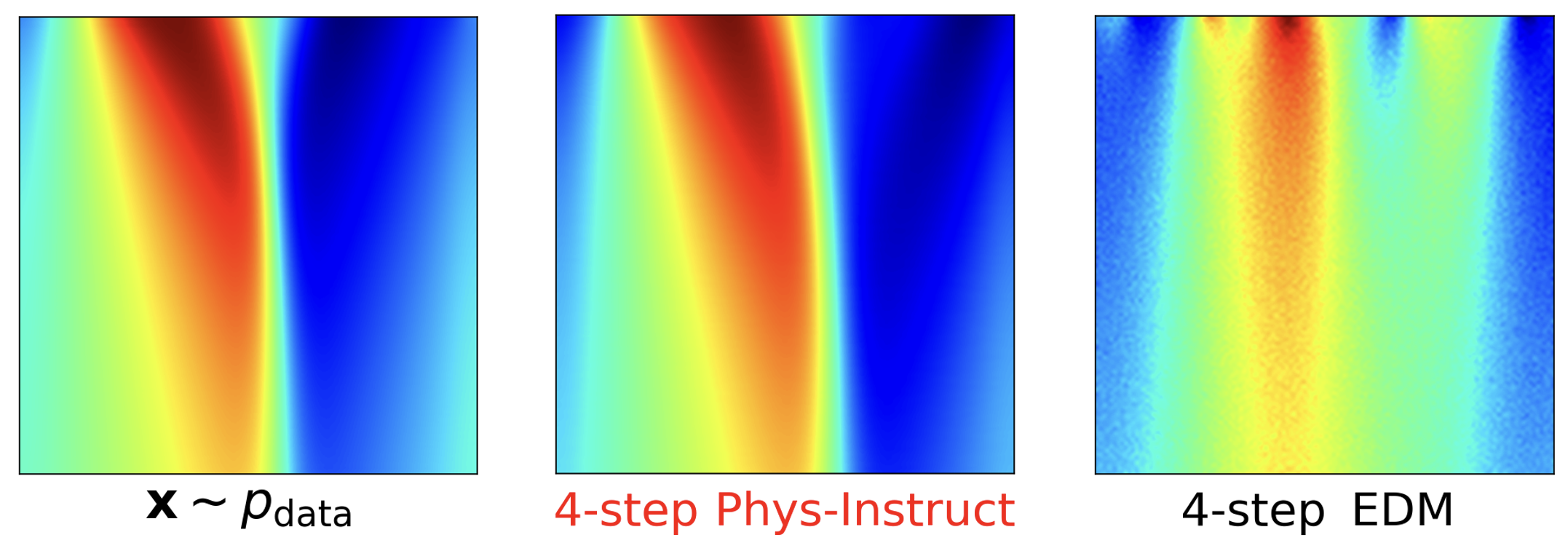}
    \end{subfigure}
    \captionsetup{font=footnotesize,skip=4pt}
    \caption{Burgers' generation comparison. \textbf{Top:} PDE Error vs.\ latency. Baselines use 4/10/100/200 steps; Phys-Instruct uses 1--4 steps. \textbf{Bottom:} Burgers' samples (left to right): data, Phys-Instruct, and EDM teacher.}
    \label{fig:burgers_latency_and_samples}
    \vspace{-4pt}
\end{wrapfigure}

Diffusion models have recently shown strong empirical performance for modeling spatially structured scientific fields, and have been increasingly explored as flexible priors for \textit{{Partial Differential Equation (PDE)}} related data \cite{daras2024survey,zheng2025inversebench}.
In PDE solving problems, such priors are appealing not only for unconditional field generation, but also for conditional tasks such as forward prediction \cite{raissi2019physics}, inverse problems \cite{cai2021physics,zheng2024hompinns}, and reconstruction from partial observations \cite{huang2024diffusionpde}.
However, making diffusion models practical for PDE tasks requires addressing two coupled issues: sampling efficiency and physical consistency.

\textbf{Sampling Efficiency:} High-fidelity samples typically require hundreds to thousands of denoising steps, resulting in large numbers of function evaluations (NFEs) and high sampling cost \cite{song2021score,ho2020denoising}. For example, DiffusionPDE \cite{huang2024diffusionpde} addresses forward, inverse, and reconstruction PDE problems from sparse observations by running a long reverse diffusion trajectory on the order of thousands of steps; \citet{yang2023denoising} predicts nonlinear fluid fields with a denoising diffusion model that samples through hundreds of iterative Langevin updates. These approaches demonstrate the potential of diffusion-based PDE modeling, but the required number of function evaluations makes inference expensive and limits their practicality in real-world settings. \textbf{Physical Consistency:} Standard diffusion training is primarily data-driven and does not explicitly enforce PDE governing equations; as a result, samples can be visually plausible yet exhibit non-negligible PDE error \cite{bastek2025physics}.
Moreover, under aggressive acceleration or data-free distillation, small distribution-matching errors can translate into noticeable physics violations, motivating explicit PDE-aware objectives \cite{zhang2025physics}.

Existing physics-guided remedies tend to trade one challenge for the other:
Some impose physics during inference via guidance or iterative correction, which can improve physical consistency but comes at the cost of higher sampling latency~\citep{huang2024diffusionpde}.
Others incorporate physics guidance during diffusion training by penalizing PDE error along multi-step denoising trajectories, but sampling typically remains iterative, and the additional constraints can introduce nontrivial trade-offs between sample quality and physical fidelity \cite{bastek2025physics,baldan2025flow}.
In contrast, distillation and other aggressive acceleration schemes reduce sampling steps \cite{luo2023diff,song2023consistency,liu2023flow,salimans2022progressive}, but physical consistency may degrade without explicit PDE equation constraint.
This is particularly pronounced when the budget is limited, where fewer sampling steps imply more aggressive updates along the sampling trajectory, increasing approximation errors when physical constraints are enforced during inference. Together, these highlight a speed-physics trade-off in diffusion-based PDE generation.

To break the fundamental trade-off between sampling efficiency and physical consistency, we introduce \textbf{Phys-Instruct}, a principled physics-guided distillation framework. Unlike standard distillation methods \cite{song2023consistency,Yin2024DMD,luo2023diff} that solely focus on distribution matching, Phys-Instruct reformulates distillation as a constrained optimization problem. It compresses a multi-step diffusion teacher into a flexible few-step generator while rigorously enforcing physical laws via a differentiable PDE-error penalty. Crucially, this explicit physics injection occurs entirely during training and requires no paired ground-truth solutions. At inference time, the distilled student model generates highly consistent PDE fields in a fixed few-step forward pass, completely bypassing the need for expensive test-time guidance, iterative corrections, or post-hoc refinements.
Extensive evaluations across five diverse PDE benchmarks demonstrate the superiority of Phys-Instruct. It achieves a remarkable dual victory: accelerating inference by orders of magnitude while reducing the PDE error by more than 8$\times$ compared to state-of-the-art multi-step diffusion baselines. Furthermore, the distilled unconditional generator serves as a highly compact and versatile diffusion prior. By seamlessly integrating it with a ControlNet-style conditional branch \cite{zhang2023adding}, Phys-Instruct can be repurposed for various challenging downstream tasks—including forward prediction, inverse problems, and partial observation reconstructions—enabling accurate, physically faithful solving with unprecedented efficiency.

Our \textbf{main contributions} are three-fold:
(i) we propose {Phys-Instruct}, a training-time physics-guided distillation framework that compresses a multi-step diffusion teacher into a one-/few-step generator for PDEs with no test-time physics guidance or control steps;
(ii) we demonstrate that incorporating PDE error as a constraint improves distillation over the no physics guidance variant, while requiring no paired ground-truth training data supervision during distillation;
(iii) We show the distilled unconditional generator serves as a reusable diffusion prior for downstream conditional PDE tasks (forward/inverse problem with full/partial observations), with favorable accuracy--latency trade-offs.

\section{Related Work}

\paragraph{Diffusion Priors for PDE Fields}

Diffusion models have emerged as powerful tools for modeling unconditional PDE distributions \citep{yuan2025} and solving conditional tasks such as forward prediction, inverse problems, and partial observation reconstruction \citep{shysheya2024conditional,baldassari2023conditional,mollaali2026padam}. However, rigorously incorporating governing physical laws remains a fundamental challenge. Existing approaches are either purely data-driven \citep{dasgupta2025conditional,wang2023generative} or attempt to explicitly inject physics through inference-time guidance \citep{huang2024diffusionpde,yao2025fundps,jacobsen2025cocogen} or training-time regularization \citep{bastek2025physics,shu2023physics}. However, these physics-aware paradigms suffer from an inherent speed-physics trade-off: inference-time guidance severely exacerbates the already sluggish multi-step sampling process, whereas training-time regularization still relies on iterative sampling and struggles to maintain physical fidelity under highly constrained computational budgets. To break this dilemma, our work establishes a novel paradigm by incorporating physics explicitly and exclusively at \emph{distillation time}. By unifying distribution matching with physics-informed learning, we distill a multi-step teacher into a highly efficient one-/few-step generator, ensuring rigorous physical consistency without relying on expensive test-time physics guidance or post-hoc iterative corrections.

\paragraph{Few-Step Diffusion and Distillation.}
Diffusion sampling is typically implemented by iteratively solving reverse-time dynamics (e.g., via SDE/ODE solvers), where high-fidelity generation often requires hundreds of computationally expensive function evaluations \citep{song2021score,Karras2022Elucidating}. To alleviate this burden, a broad line of work accelerates sampling either through advanced training-free numerical solvers \citep{lu2022dpmsolverplusplus,dockhorn2022genie,zhou2024fast,nichol2021improved,shih2023parallel} or by distilling multi-step teachers into efficient one-/few-step generators \citep{salimans2022progressive,liu2023flow}. Within the distillation paradigm, prominent approaches include consistency models, which enforce trajectory consistency across noise levels \citep{song2023consistency,song2023improved,lu2024simplifying,Geng2024ConsistencyModelsMadeEasy,Luo2023LatentConsistencyModels,Kim2024CTM,Geng2025MeanFlows}, and distribution-matching objectives rooted in Diff-Instruct \citep{luo2023diff}, which have inspired a wide array of fast-generation extensions \citep{wang2025uni,luo2024diffstar,Luo2024SIM,xu2025one,Yin2024DMD,Yin2024DMD2,Xie2024EMD,fan2023dpok,Zhou2024ScoreIdentity,huang2024flow,yoso,luo2025reward,zheng2025ultra,wu2026didr}.
While well-established in image and text generation, extending diffusion distillation to PDE fields requires preserving physical consistency alongside distributional fidelity. Recent efforts~\citep{zhang2025physics}, addresses this challenge by incorporating PDE residual penalties on generated samples during distillation. Phys-Instruct approaches this problem through a penalized integral KL formulation that combines teacher–student distribution matching with differentiable PDE constraints. This formulation supports flexible one-/few-step generation without inference-time physics guidance or iterative correction, while yielding a reusable generative prior for downstream conditional PDE tasks.

\section{Problem Formulation and Diffusion Preliminaries}
\subsection{PDE Field Generation Setup}
We consider PDE problems on a domain $\mathcal{D}$, where $\mathcal{D}=\Omega_x\subset\mathbb{R}^d$ for static systems and $\mathcal{D}=\Omega_x\times\Omega_t$ for dynamic systems. Let $\xi\in\mathcal{D}$ denote the coordinates. We collect all physical fields into $\mathbf{x}(\xi)\in\mathbb{R}^c$ and partition it as
$\mathbf{x}(\xi)=\big[\mathbf{a}(\xi)^\top,\mathbf{u}(\xi)^\top\big]^\top$,
where $\mathbf{a}(\xi)\in\mathbb{R}^{c_1}$, $\mathbf{u}(\xi)\in\mathbb{R}^{c_2}$, and $c=c_1+c_2$.
The pair $(\mathbf{a},\mathbf{u})$ specifies a PDE instance and its associated solution: in static problems, $\mathbf{a}$ represents coefficients and $\mathbf{u}$ is the corresponding solution field; in dynamic problems, $\mathbf{a}$ specifies the initial condition at $t=0$ and $\mathbf{u}:=\mathbf{u}(\cdot,T)$ denotes the solution state at a target time $T$.
The governing laws, together with boundary/initial constraints, are written as
\begin{equation}
  \mathcal{G}[\mathbf{x}](\xi)=0,\qquad \xi\in\Omega,
  \label{eq:continuous_pde}
\end{equation}
where $\mathcal{G}$ is a differential operator.
Our goal is to generate $\mathbf{x}=(\mathbf{a},\mathbf{u})$ from the generator $g_\theta$ that satisfy \eqref{eq:continuous_pde}.

\subsection{Diffusion Models}
Assume we observe data from an unknown distribution $\mathbf{x} \sim p_{\text{data}}(\mathbf{x})$.
Diffusion models construct a forward noising process that transforms the initial distribution $p_0 =p_{\text{data}} $ into a simple noise distribution.
Specifically, the forward process $\{\mathbf{x}_t\}_{t\in[0,T]}$ is described by the It\^o SDE:
\begin{equation}
\label{eq:fwd-process}
  d\mathbf{x}_t = \mathbf{F}(\mathbf{x}_t,t)\,dt + G(t)\,d\mathbf{w}_t,
\end{equation}
where $\mathbf{F}$ and $G(t)$ are pre-defined drift and diffusion coefficients, and $\mathbf{w}_t$ is a standard Wiener process.
We write $p_t(\mathbf{x}_t)$ for the marginal density of $\mathbf{x}_t$ and $p_{0t}(\mathbf{x}_t\mid \mathbf{x}_0)$ for the forward transition density induced by~\eqref{eq:fwd-process}.

To reverse the noising dynamics, score-based diffusion models estimate the time-dependent score $\nabla_{\mathbf{x}_t}\log p_t(\mathbf{x}_t)$.
Since it is intractable in high dimensions, we train a time-conditioned network $s_\phi(\mathbf{x}_t,t)$ to approximate it across noise levels.
In practice, $s_\phi$ is implemented as a discrete noise-level network~\cite{song2019generative,ho2020denoising} or a continuous-time network~\cite{Karras2022Elucidating,song2021score}.
Training is performed
via minimizing a weighted denoising score matching objective~\citep{song2021score}:
\begin{equation}
\mathcal{L}_{\mathrm{DSM}}(\phi) = \int_{0}^{T} w(t)\, \mathbb{E}_{\mathbf{x}_0,\mathbf{x}_t \mid \mathbf{x}_0} \left\| s_{\phi}(\mathbf{x}_t, t) - \nabla_{\mathbf{x}_t} \log p_{0t}(\mathbf{x}_t \mid \mathbf{x}_0) \right\|_2^2 \mathrm dt ,
\label{eq:dsm}
\end{equation}
where $\mathbf{x}_0 \sim p_0$, and $\mathbf{x}_t \mid \mathbf{x}_0 \sim p_{0t}(\mathbf{x}_t \mid \mathbf{x}_0)$, the weighting function $w(t)$ balances the contributions from different noise levels.

\subsection{Integral Kullback-Leibler Divergence for Diffusion Distillation}
Reverse \eqref{eq:fwd-process} via iterative sampling typically requires a large number of sampling steps, leading to high inference latency, which can limit the practicality of diffusion models for PDE solving. This bottleneck can be alleviated by casting accelerated generation as a distillation problem: a multi-step teacher provides high-quality generative behavior, while a student is distilled from the teacher to produce comparable outputs in substantially fewer function evaluations. To encourage the student to learn the teacher’s evolution over time rather than only matching the final output, we adopt a distributional distillation framework that aligns their time-indexed marginals along the diffusion process.

Let $\{p_t(\mathbf{x}_t)\}_{t\in[0,T]}$ denote the teacher marginals induced by the forward diffusion \eqref{eq:fwd-process}, and let the teacher score be given by a frozen network
$s_\nu(\mathbf{x}_t,t)\approx \nabla_{\mathbf{x}_t}\log p_t(\mathbf{x}_t)$.
The student model induces a distribution $q_{\theta,0}$ over clean samples $\mathbf{x}_0$ (e.g., through a few-step generator), and we define its time-$t$ marginals $\{q_{\theta,t}(\mathbf{x}_t)\}_{t\in[0,T]}$ by applying the same forward diffusion \eqref{eq:fwd-process} to $\mathbf{x}_0\sim q_{\theta,0}$, i.e.,
$\mathbf{x}_t\mid \mathbf{x}_0 \sim p_{0t}(\mathbf{x}_t\mid \mathbf{x}_0)$.
Integral Kullback-Leibler (IKL) divergence generalizes the standard KL divergence by integrating the discrepancy between teacher and student marginals across all diffusion times, which was initially used in \citet{luo2023diff}:
\begin{equation}
\label{eq:ikl}
\mathcal{D}^{[0,T]}_{\mathrm{IKL}}(q_\theta\|p)
:= \int_{0}^{T} w(t)\, \mathcal{D}_{\mathrm{KL}}\!\big(q_{\theta,t}\,\|\,p_t\big)\,\mathrm{d}t \\
= \int_{0}^{T} w(t)\,
\mathbb{E}_{\mathbf{x}_t \sim q_{\theta,t}}
\left[\log \frac{q_{\theta,t}(\mathbf{x}_t)}{p_t(\mathbf{x}_t)}\right]\mathrm{d}t,
\end{equation}
where $w(t)>0$ weights the relative importance of different noise levels.
Minimizing $\mathcal{D}^{[0,T]}_{\mathrm{IKL}}$ promotes trajectory-level consistency between the student and teacher models is essential for high-quality sampling with limited budgets.

\begin{figure}[htbp]
    \centering
    \includegraphics[width=.999\linewidth]{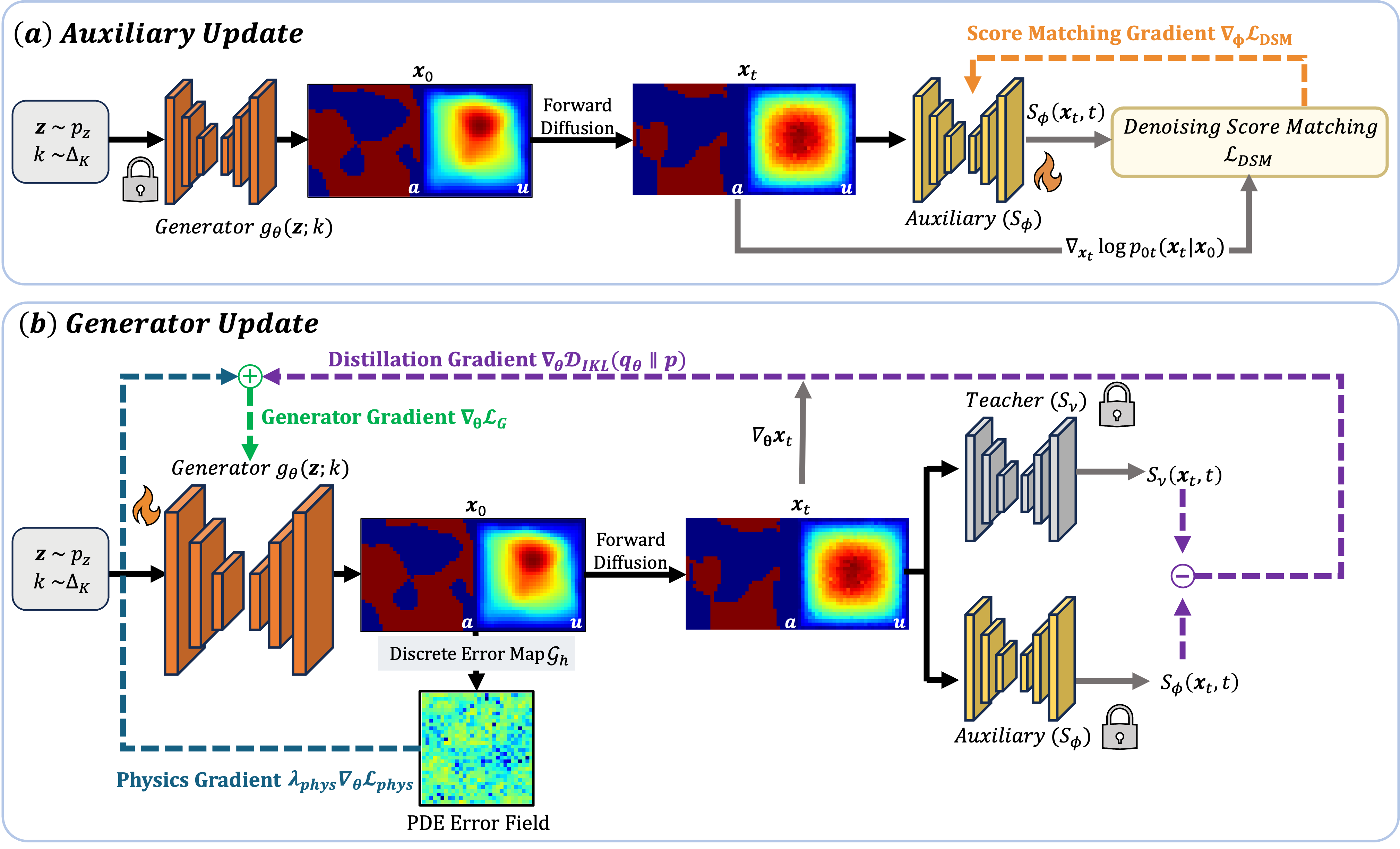}
\caption{
\textbf{Overview of Phys-Instruct distillation.}
Training alternates between two phases with sampled latent
variables $\mathbf{z}\sim p_z$ and step budgets $k\sim\Delta_K$.
\textbf{(a) Auxiliary update.}
With the generator $g_\theta$ held fixed, generated samples
$\mathbf{x}_0=g_\theta(\mathbf{z};k)$ are detached and forward-diffused
to $\mathbf{x}_t$. The auxiliary network $s_\phi$ receives
$(\mathbf{x}_t,t)$ and is updated by denoising score matching using
the target $\nabla_{\mathbf{x}_t}\log p_{0t}(\mathbf{x}_t\mid\mathbf{x}_0)$.
\textbf{(b) Generator update.}
The frozen teacher $s_\nu$ and the auxiliary network $s_\phi$, held
fixed during this phase, evaluate the same noisy input $(\mathbf{x}_t,t)$.
Their score difference is treated as a stop-gradient signal to compute
the IKL distillation gradient. This gradient is combined with
$\lambda_{\mathrm{phys}}\nabla_\theta\mathcal{L}_{\mathrm{phys}}$
to update $g_\theta$, with the physics guidance evaluated on the
clean generated sample $\mathbf{x}_0$.}

    \label{fig:physinstruct}
\end{figure}

\section{A Unified Framework for Physics-Guided Few-Step Distillation}
While the IKL objective in \eqref{eq:ikl} provides a principled way to distill a multi-step diffusion teacher into a few-step student, matching the teacher distribution alone does not ensure that the student outputs satisfy the target PDE constraints. Existing physics-guided remedies often enforce feasibility at inference time through guidance or iterative correction, which increases the number of function evaluations and erodes the latency gains of few-step generation.

We therefore propose \textbf{Phys-Instruct}, a diffusion distillation framework that \emph{instructs} a few-step student with physics guidance during training. Concretely, we distill a pretrained multi-step diffusion teacher into a \emph{few-step} generator while simultaneously enforcing the PDE constraint \eqref{eq:continuous_pde} through a differentiable PDE error computed on generated samples. In this way, physics enters as explicit training-time supervision for the student, rather than as test-time guidance. As a result, Phys-Instruct produces PDE-consistent samples with a few-step sampler and requires no paired ground-truth solutions during distillation. Figure~\ref{fig:physinstruct} provides a schematic overview of the proposed framework.

\subsection{Flexible Few-Step Sampling with Phys-Instruct}
\label{subsec:physinstruct_generator}
\citet{luo2023diff} distills a pretrained diffusion teacher into an \emph{implicit one-step} sampler, mapping a Gaussian latent $\mathbf{z}$ to a sample in a single network evaluation. While maximally fast, a single step can be overly restrictive for approximating the teacher's multi-step sampling trajectory.
Phys-Instruct generalizes this idea by enabling the student sampler to unroll a \emph{small} number of sampling steps, yielding a distilled diffusion model which can take on step budget $k \in \{1,\ldots,K\}$. This provides a simple knob to trade sampling quality for compute while preserving teacher-aligned sampling semantics. Specifically, we parameterize the student generator as a diffusion model $g_\theta$
and denote its output with $k$ unrolled steps by $\mathbf{x}_0 = g_\theta(\mathbf{z};k)$, where the latent variable $\mathbf{z}\sim p_z$.

\paragraph{Unified $k$-Step Formulation.}
For a given step budget $k$ and fixed noise level $\sigma_{\mathrm{init}}$, we select $k{+}1$ noise levels from a fixed sampling schedule,
$
\sigma_k > \sigma_{k-1} > \cdots > \sigma_0$,
where $\sigma_k = \sigma_{\mathrm{init}}$,
and draw the initial state
$\mathbf{x}_k := \mathbf{z}$, with
$\mathbf{z}\sim \mathcal{N}\!\left(\mathbf{0},\sigma_{\mathrm{init}}^2\mathbf{I}\right)$.
We then iteratively apply a student-aligned step operator for $k$ steps,
$
\mathbf{x}_{i-1} = \mathsf{Step}\!\left(\mathbf{x}_i;\sigma_i,\sigma_{i-1}, g_\theta\right)$, $i=k,k-1,\ldots,1,
$
where $\mathsf{Step}(\cdot)$ belongs to the sampler family that the student diffusion model structure supports (e.g., the Euler \citep{ascher1998computer} or Heun \citep{Karras2022Elucidating} sampler), and is parameterized by the student model $g_\theta$.

\subsection{PDE Error and Physics Loss}
\label{sec 4.2}
In practice, all physical fields are defined on a fixed discretized grid
$\mathcal{D}_h=\{\xi_i\}_{i=1}^{N}\subset\Omega$.
We approximate the continuous constraint $\mathcal{G}$ in \eqref{eq:continuous_pde} on this grid with a differentiable \emph{discrete PDE error map} $\mathcal{G}_h$,
implemented using finite-difference stencils or convolutional filters.
The resulting \emph{PDE error (residual) field} is
\begin{equation}
  r(\mathbf{x})_i := \mathcal{G}_h[\mathbf{x}](\xi_i),\qquad i=1,\dots,N,
\end{equation}
and we summarize it into a scalar physics error via the mean squared PDE error
\begin{equation}
  \mathcal{R}(\mathbf{x})
  := \frac{1}{N}\sum_{i=1}^{N} \bigl\|r(\mathbf{x})_i\bigr\|_2^2.
  \label{eq:phys_loss}
\end{equation}
When boundary conditions are prescribed, we impose them explicitly and compute the error only over interior (unconstrained) nodes, which reduces spurious errors introduced by boundary discretization.

\subsection{Distillation Objective with Physical Constraints}
\label{subsec:onestep_physinstruct}
A central challenge in physics-guided generative modeling is relating the generator's update rule to a rigorous distribution-matching objective while simultaneously adhering to physical laws. Motivated by recent advances that successfully incorporate hard boundaries and explicit constraints into complex sampling and non-convex optimization tasks \cite{zhang2022sampling,lamperski2021projected}, we propose to reformulate the standard distillation process as a formal constrained optimization problem. Specifically, our goal is to minimize the IKL divergence between the student and teacher distributions, subject to a physical feasibility constraint enforced by the PDE error.

Starting from IKL minimization in \eqref{eq:ikl}, we incorporate the physics guidance via \eqref{eq:phys_loss}. Conceptually, this corresponds to a \emph{constrained} distillation problem that matches the teacher distribution while enforcing physical feasibility\footnote{We rewrite $\mathcal{D}^{[0,T]}_{\mathrm{IKL}}(\cdot \|\, \cdot )$ as $\mathcal{D}_{\mathrm{IKL}}(\cdot  \,\|\, \cdot )$ unless stated otherwise.} \cite{liu2021sampling,deng2024reflected}:
\begin{equation}
\min_{\theta}\ \mathcal{D}_{\mathrm{IKL}}(q_\theta \,\|\, p)
\quad \text{s.t.}\quad
\mathbb{E}_{\mathbf{x}_0\sim q_\theta}\!\left[\mathcal{R}(\mathbf{x}_0)\right]\le \varepsilon,
\label{eq:constrained_ikl}
\end{equation}
where $\mathcal{R}(\mathbf{x}_0)\ge 0$ denotes the PDE error evaluated on a generated sample $\mathbf{x}_0$ by the student, $\varepsilon > 0$ is some small threshold. In our PDE setting, each $\mathbf{x}_0$ encodes the solution field together with the coefficient/forcing fields, and $\mathcal{R}(\mathbf{x}_0)$ is computed with \eqref{eq:phys_loss}.

Building on the unified $k$-step formulation in Sec.~\ref{subsec:physinstruct_generator}, we train a \emph{single} student diffusion model that supports flexible sampling budgets $k\in\{1,\ldots,K\}$ by randomly unrolling $k$ steps during training. Concretely, we draw
\[
k \sim \Delta_K \quad \text{over } \{1,\ldots,K\},
\]
where $\Delta_K$ is a categorical distribution over step budgets, and generate a student sample via the $k$-step generator $\mathbf{x}_0 = g_\theta(\mathbf{z};k)$ with $\mathbf{z}\sim p_z$.

We define the physics guidance $\mathcal{L}_{\mathrm{phys}}$ on student samples:
\begin{equation}
\mathcal{L}_{\mathrm{phys}}(\theta)
:=\mathbb{E}_{\mathbf{z}\sim p_z,\,k\sim \Delta_K,\mathbf{x}_0=g_\theta(\mathbf{z};k)}\Big[
\mathcal{R}\big(\mathbf{x}_0\big)
\Big].
\label{eq:physics_loss}
\end{equation}
This term is applied \emph{only} when updating the student generator: the teacher score network and the auxiliary score model are not directly regularized by $\mathcal{L}_{\mathrm{phys}}$, preserving the original IKL distillation dynamics. It is worth noting that we evaluate the PDE error on generated final sample $\mathbf{x}_0$, rather than on the mean prediction such as $\mathbb{E}[\mathbf{x}_0\mid \mathbf{z}]$. This distinction is critical for nonlinear PDEs: optimizing the mean incurs a Jensen gap that conflicts with distribution matching \citep{zhang2025physics}. Penalizing the PDE error on individual samples directly helps avoid this bias and ensures that generated samples adhere to physical laws.

However, directly solving constrained optimization \eqref{eq:constrained_ikl} in high-dimensional deep generative models is computationally intractable. To overcome this, we adopt a penalty relaxation strategy, translating the constrained formulation into a penalized IKL objective where the original distillation is regularized by the physics guidance term in \eqref{eq:physics_loss}. The following theorem formalizes this penalized objective and provides a score-function identity that ensures tractable optimization. Crucially, it demonstrates that the gradient of our objective elegantly decomposes into two distinct components: an IKL distribution-matching term in the spirit of Diff-Instruct, and a physics-informed gradient term that explicitly steers the generator toward PDE satisfaction.

\begin{theorem}[Physics-Guided Score-Function Identity]
\label{thm:phys_impl_one}
Consider the physics-constrained IKL matching problem~\eqref{eq:constrained_ikl}.
As a practical surrogate, we apply a Lagrangian relaxation of \eqref{eq:constrained_ikl} and define the generator objective, treating  $\lambda_{\mathrm{phys}}>0$ as a trade-off hyperparameter:
\begin{equation}
\mathcal{L}_{G}(\theta)
:=
\mathcal{D}_{\mathrm{IKL}}(q_\theta\|p)
+\lambda_{\mathrm{phys}}\,\mathcal{L}_{\mathrm{phys}}(\theta),
\label{eq:physinstruct_loss}
\end{equation}
where $\mathcal{L}_{\mathrm{phys}}(\theta)$ is given by~\eqref{eq:physics_loss}.
Under mild regularity conditions stated in Appendix~\ref{apx:proof}, the gradient of the implemented generator objective $\mathcal{L}_{G}(\theta)$ admits the identity
\begin{equation}
\label{eq:phys_main_one}
\begin{aligned}
\nabla_\theta \mathcal{L}_{G}(\theta)
&=
\int_{0}^{T} w(t)\,
\mathbb{E}_{\substack{\mathbf z\sim p_{\mathbf z},\, k\sim\Delta_K\\
\mathbf x_0=g_\theta(\mathbf z;k),\,
\mathbf x_t\sim p_{0t}(\cdot\mid \mathbf x_0)}}
\!\left[
\mathrm{SG}\!\left\{s_\phi(\mathbf x_t,t)-s_\nu(\mathbf x_t,t)\right\}^{\!\top}
\nabla_\theta \mathbf x_t
\right]\mathrm dt  \\
&\quad
+\lambda_{\mathrm{phys}}\,
\mathbb{E}_{\mathbf z\sim p_{\mathbf z},\, k\sim\Delta_K}
\!\left[
\big(\nabla_{\mathbf x_0}\mathcal R(\mathbf x_0)\big)^{\!\top}
\nabla_\theta g_\theta(\mathbf z;k)
\right]_{\mathbf x_0=g_\theta(\mathbf z;k)},
\end{aligned}
\end{equation}
where $\mathrm{SG}(\cdot)$ denotes the stop-gradient operator, and
$s_\phi(\mathbf{x}_t,t)$ is an auxiliary model trained to approximate the marginal score
$\nabla_{\mathbf{x}_t}\log q_{\theta,t}(\mathbf{x}_t)$ of the student-induced distribution at time $t$.
Consequently, the gradient update implemented by~\eqref{eq:phys_main_one} yields a stochastic estimator of the
gradient of the penalized IKL objective with PDE guidance.
\end{theorem}

We defer the detailed proof to Appendix~\ref{apx:proof} for interested readers.

\begin{remark}[Amortized physics correction]
\label{rem:amortized_phys}
The second term in~\eqref{eq:phys_main_one} provides an explicit physics correction direction for the
student generator.
Indeed, for a fixed $(\mathbf z,k)$, the chain rule gives
\[
\nabla_\theta \mathcal{R}\!\big(g_\theta(\mathbf z;k)\big)
=
\Big(\nabla_{\mathbf x_0}\mathcal{R}(\mathbf x_0)\Big)^{\!\top}\nabla_\theta g_\theta(\mathbf z;k)
\Big|_{\mathbf x_0=g_\theta(\mathbf z;k)}.
\]
Thus, the physics-guided component in~\eqref{eq:phys_main_one} pushes the generator outputs along
$-\nabla_{\mathbf x_0}\mathcal{R}(\mathbf x_0)$, i.e., toward lower-residual samples.
This can be viewed as amortizing a physics feasibility-improving correction into the student during training,
thereby reducing the reliance on additional test-time guidance when operating at small step budgets.
\end{remark}

\subsection{Phys-Instruct Training Procedure}
The Phys-Instruct distillation procedure assumes a pretrained diffusion model $s_\nu$ is given. The teacher diffusion models are trained with PDE data to reverse \eqref{eq:fwd-process} that progressively adds noise to a sample from a real data distribution $p_0$, where $R(\mathbf{x}_0^*)=0$ if $\mathbf{x}_0^* \sim p_0$. The teacher diffusion model $s_\nu$ is frozen for the whole training procedure.

Phys-Instruct trains the student generator $g_\theta$ through two alternative phases between updating the auxiliary diffusion model $s_\phi$, and updating the student generator $g_\theta$, with both $g_\theta$ and $s_\phi$ initialized with the teacher's parameters. The former phase follows the standard diffusion model learning procedure, i.e., minimizing loss function \eqref{eq:dsm}, with a slight change that $\mathbf{x}_0$ is not from the training dataset but generated from $g_\theta$. The auxiliary model $s_\phi$ is updated online to track the score $\nabla_{\mathbf{x}_t}\log q_{\theta,t}(\mathbf{x}_t)$ of the current student marginal at noise level $t$. The latter phase updates the generator’s parameter $\theta$ using $\mathcal{L}_{G}(\theta)$, whose tractable gradient $\nabla_\theta \mathcal{L}_{G}(\theta)$ is derived as \eqref{eq:phys_main_one}.
The whole procedure is summarized in Algorithm \ref{alg:physinstruct}.

\section{Experiments}

\textbf{PDE Benchmarks.} We conduct experiments on 5 different PDE benchmarks, including Darcy flow, the Poisson equation, the non-bounded Navier-Stokes equation, the Burgers' equation, and the inhomogeneous Helmholtz equation. The dataset generation mainly follow released scripts in \cite{huang2024diffusionpde} and \cite{li2021fourier}. Different benchmarks are of different resolutions as indicated in Table \ref{tab:uncond}. More details about the PDE benchmarks are deferred to Appendix \ref{apx:pde}.

\textbf{Training.}
For all benchmarks, Phys-Instruct utilizes the EDM teacher model \cite{Karras2022Elucidating} and jointly generates both the coefficient field $\mathbf{a}$ and the solution field $\mathbf{u}$, except for the Burgers' equation benchmark, where we store only the space–time field $\mathbf{u}$. Teacher models are trained with 50,000 data points. The teacher models for Darcy flow, Poisson, and non-bounded Navier-Stokes are trained with 4 V100 GPUs for around 6 hours. The teacher models for Burgers' and Helmholtz come from \cite{huang2024diffusionpde}.
The step operator follows the deterministic second-order Heun sampler in~\cite{Karras2022Elucidating}.
The distillation is conducted with max sampling step $K=4$ with 2 A100 GPUs for around 4 hours. A normalized compute comparison and an update-level cost breakdown are provided in Appendix~\ref{app:compute_cost}.

\textbf{Baselines and Protocols.}  We evaluate four diffusion-based baselines using the deterministic second-order Heun sampler following ~\cite{Karras2022Elucidating}. To assess both generation quality and computational efficiency, we report results at 200 steps representing the baselines' full potential, and at 4 steps matching the budget of our proposed method.
Our comparison includes the plain EDM teacher model ~\cite{Karras2022Elucidating} as a purely data-driven reference without physical constraints. We then consider two methods that enforce physics during sampling: DiffusionPDE~\cite{huang2024diffusionpde} augments EDM by injecting physics guidance specifically during the final 20\% of the sampling trajectory following diffusion posterior sampling (DPS) \citep{chung2023diffusion}, with this guidance extending to the last 50\% when step=4. Similarly, CoCoGen~\cite{jacobsen2025cocogen} starts from the DiffusionPDE output and performs an additional ten-step physics-guided refinement, or two extra steps specifically for the four-step setting. Finally, we compare against PIDM~\cite{bastek2025physics}, which distinguishes itself by incorporating physics guidance directly into the training objective rather than during sampling.

\begin{table}[tb]
\centering
\small
\caption{
Unconditional generation results on PDE benchmarks.
We report PDE Error$\downarrow$ and the number of sampling steps per sample (Step).
NS$_{\mathrm{unb.}}$ denotes the nonbounded Navier--Stokes benchmark, and the lowest PDE Error for each benchmark is \textbf{bolded}.
}
\label{tab:uncond}
\adjustbox{center=\textwidth,scale=0.90}{\begin{tabular}{lcccccc}
\toprule
& & \multicolumn{5}{c}{PDE Error ($\downarrow$)} \\
\cmidrule(lr){3-7}
Method & Step & Darcy Flow & Poisson & $\text{NS}_{\text{unb.}}$ & Burgers' & Helmholtz \\
\midrule
\multicolumn{2}{l}{Resolution}
  & $32{\times}32{\times}2$
  & $32{\times}32{\times}2$
  & $64{\times}64{\times}2$
  & $128{\times}128{\times}1$
  & $128{\times}128{\times}2$\\
\midrule

\multirow{4}{*}{Phys-Instruct}
 & 1
  & $3.15{\times}10^{-1}$
  & $2.16{\times}10^{-2}$
  & $3.15{\times}10^{-2}$
  & $1.71{\times}10^{-1}$
  & $3.75{\times}10^{-1}$ \\
 & 2
  & $2.18{\times}10^{-1}$
  & $5.18{\times}10^{-3}$
  & $1.46{\times}10^{-2}$
  & $1.18{\times}10^{-1}$
  & $3.00{\times}10^{-1}$ \\
 & 3
  & $1.35{\times}10^{-1}$
  & $4.79{\times}10^{-3}$
  & $3.65{\times}10^{-3}$
  & $7.17{\times}10^{-2}$
  & $2.10{\times}10^{-1}$ \\
 & 4
  & $\mathbf{8.49{\times}10^{-2}}$
  & $\mathbf{3.99{\times}10^{-3}}$
  & $\mathbf{3.21{\times}10^{-3}}$
  & $\mathbf{5.54{\times}10^{-2}}$
  & $\mathbf{1.46{\times}10^{-1}}$ \\
\midrule

EDM \cite{Karras2022Elucidating} & 4
  & $196.03$
  & $16.72$
  & $3.33{\times}10^{-1}$
  & $144.23$
  & $79.77$\\
DiffusionPDE \cite{huang2024diffusionpde} & 4
  & $191.50$
  & $16.68$
  & $3.30{\times}10^{-1}$
  & $141.37$
  & $78.97$\\
CoCoGen \cite{jacobsen2025cocogen} & 4
  & $190.87$
  & $16.60$
  & $3.24{\times}10^{-1}$
  & $135.65$
  & $77.34$\\
PIDM \cite{bastek2025physics} & 4
  & $101.04$
  & $22.64$
  & $3.44{\times}10^{-1}$
  & $97.76$
  & $77.86$\\
\midrule

EDM \cite{Karras2022Elucidating} & 200
  & $5.59{\times}10^{-1}$
  & $3.14{\times}10^{-2}$
  & $3.60{\times}10^{-2}$
  & $8.60{\times}10^{-1}$
  & $3.54{\times}10^{-1}$\\
DiffusionPDE \cite{huang2024diffusionpde} & 200
  & $4.38{\times}10^{-1}$
  & $2.64{\times}10^{-2}$
  & $3.23{\times}10^{-2}$
  & $5.62{\times}10^{-1}$
  & $3.38{\times}10^{-1}$\\
CoCoGen \cite{jacobsen2025cocogen} & 200
  & $3.31{\times}10^{-1}$
  & $1.40{\times}10^{-2}$
  & $2.92{\times}10^{-2}$
  & $3.84{\times}10^{-1}$
  & $3.04{\times}10^{-1}$\\
PIDM \cite{bastek2025physics} & 200
  & $5.26{\times}10^{-1}$
  & $2.36{\times}10^{-2}$
  & $3.54{\times}10^{-2}$
  & $5.14{\times}10^{-1}$
  & $3.15{\times}10^{-1}$\\
\bottomrule
\end{tabular}}
\end{table}

\subsection{Unconditional PDE Generation Results}

As our primary physics consistency metric, we report \emph{PDE Error}$\downarrow$, the root-mean-square of the PDE error computed per sample as
$\sqrt{\mathcal{R}(\mathbf{x}_0)}$ and then averaged over 1,000 samples; the results are summarized in Table~\ref{tab:uncond}.
The efficiency trade-off is visualized in Figure~\ref{fig:burgers_latency_and_samples}, and the distributional discrepancy, measured by SWD and MMD, is reported in Figure~\ref{fig:dist_metrics_burgers}. Additional unconditional-generation results, including sample-level variability and resolution-scaling experiments, are provided in Appendix~\ref{app:additional_uncond}.

It is shown in Table~\ref{tab:uncond} that Phys-Instruct is particularly effective in the low-step regime.
Increasing the sampling steps from 1 to 4 yields consistent improvements in PDE error across all benchmarks. Figure~\ref{fig:darcy_uncond} further provides qualitative generation results of Darcy flow with 1-4 step Phys-Instruct.
Notably, the 1-step Phys-Instruct already outperforms the best Darcy flow and Burgers' baseline even when the baseline uses 200 sampling steps, and the 4-step Phys-Instruct attains the lowest PDE error overall.
In contrast, at the same low sampling budget (Step=4), diffusion baselines often produce samples with substantially larger PDE error.
With such aggressive step sizes, the correction evaluation in second-order samplers can become sensitive to model and discretization errors, which may degrade the final sample quality, as illustrated in Figure~\ref{fig:burgers_latency_and_samples}.
Phys-Instruct alleviates this issue because the student is distilled with the Heun sampler at fixed noise levels, i.e., it matches the teacher along a specific discrete trajectory.
This effectively learns a discrete transport map tailored to the few-step schedule, rather than relying on large-step integration of the teacher's continuous-time dynamics.

The efficiency advantage is further illustrated in Figure~\ref{fig:burgers_latency_and_samples}. Latency measures end-to-end sampling time and includes any physics guidance or refinement applied during sampling.
Diffusion baselines exhibit a clear compute--quality trade-off: achieving lower PDE error typically requires substantially more sampling steps, whereas very small step budgets often degrade physics consistency.
In contrast, Phys-Instruct attains low PDE error at significantly lower latency.
Increasing the student budget from Step=1 to Step=4 yields consistent improvements on Burgers', reducing PDE error from $1.71{\times}10^{-1}$ to $5.54{\times}10^{-2}$, while remaining in the low-latency regime.

\begin{figure}[htbp]
    \centering

    \includegraphics[width=0.60\linewidth]{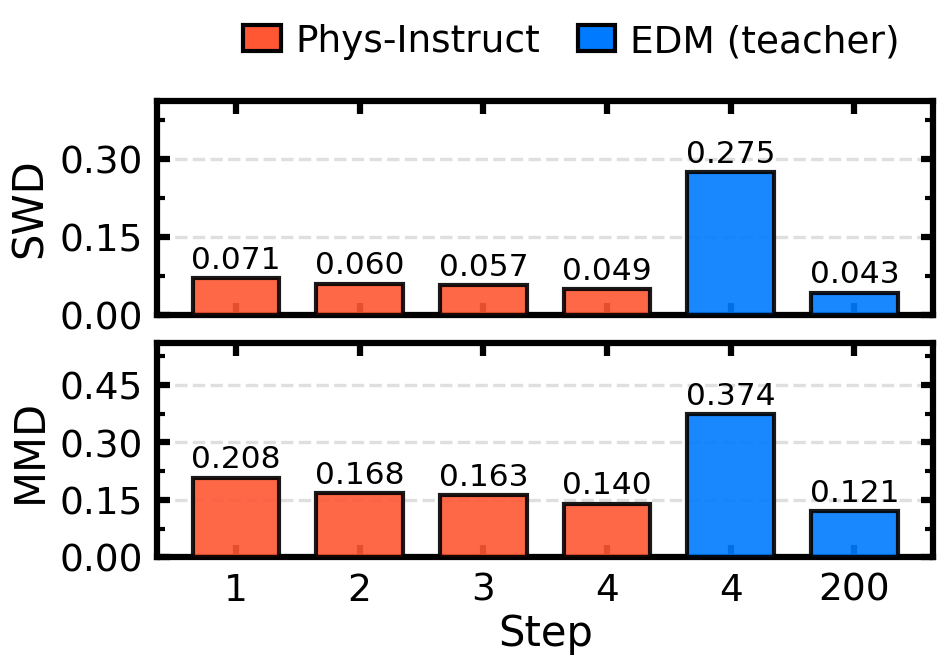}

    \caption{
    Distributional discrepancy on the Burgers' benchmark measured by SWD and MMD (reported as $\sqrt{\mathrm{MMD}^2}$).
    }
    \label{fig:dist_metrics_burgers}

\end{figure}

Beyond physics consistency, we measure distributional discrepancy to held-out data on Burgers' using SWD and MMD, reported in Figure~\ref{fig:dist_metrics_burgers}.
Few-step Phys-Instruct substantially improves both metrics compared to low-budget teacher sampling.
At Step$=4$, SWD and MMD improve by 82\% and 63\%, respectively.
Meanwhile, the four-step Phys-Instruct achieves SWD and MMD that are comparable to those of the high-budget EDM teacher, albeit slightly worse.
These results indicate effective distillation and suggest that incorporating physics guidance during training does not induce mode collapse.
The remaining gap is expected, since Phys-Instruct is distilled from the teacher under a constrained step budget rather than directly optimizing SWD/MMD to match the data distribution.

\subsection{Guidance Schedule and Strength Ablations}
\label{sec:ablation}

\begin{figure}[htbp]
    \centering

    \includegraphics[width=0.60\linewidth]{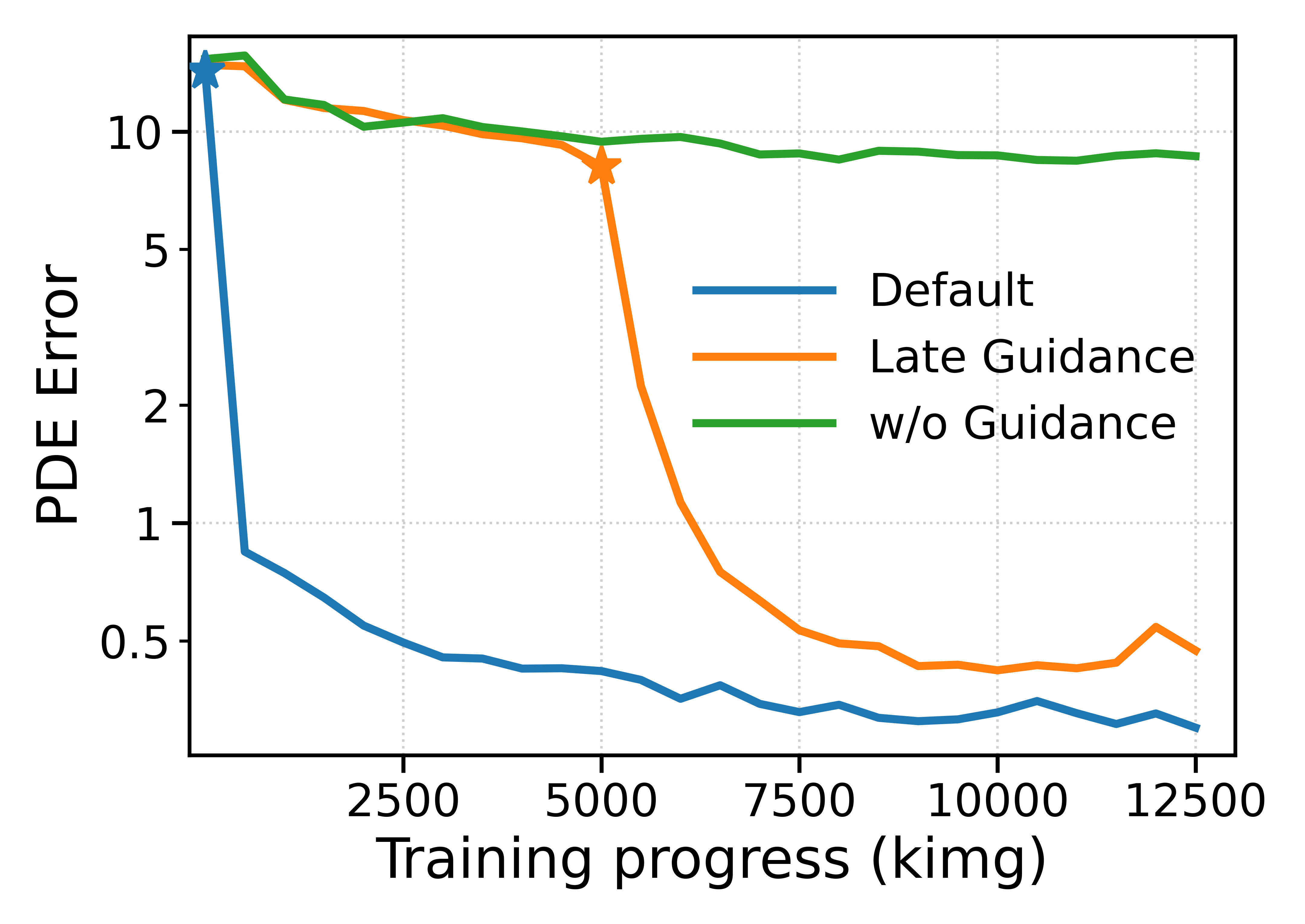}

    \caption{
    Darcy-flow PDE error during one-step Phys-Instruct distillation under different guidance schedules. $*$ indicate when guidance starts.
    }
    \label{fig:pde_tick}
\end{figure}

We ablate the role of physics guidance in Phys-Instruct to understand how it improves physical consistency.
Table~\ref{tab:ablation_main} reports results on Darcy flow at resolution $32\times32$ with fixed $\sigma_{\text{init}}=1.5$ under one-step sampling, measured by the PDE Error.
Removing PDE guidance (w/o Guidance) leads to degradation in physics consistency, with a PDE Error of $8.75$ on Darcy flow. In contrast, adding the physics guidance with the default weight ($\lambda_{\text{phys}}=5{\times}10^{-3}$) decreases PDE Error by approximately $27$ times.
Figure~\ref{fig:pde_tick} further illustrates the training dynamics.
The trajectory without physics guidance plateaus at a high PDE Error, whereas applying physics guidance from the start drives a faster and more sustained decrease in PDE Error.
When guidance is activated only after the PDE error plateaus, the error decreases after the intervention, but does not fully match the best performance.
Overall, these results indicate that PDE guidance provides an effective optimization signal and that introducing it throughout distillation yields the best result.

Additionally, Table~\ref{tab:ablation_main} indicates that the method is reasonably robust to the choice of $\lambda_{\text{phys}}$ within a practical range. Overly weak physics weighting noticeably degrades physical consistency, whereas a larger weight performs comparably to the default.
Additional sensitivity analyses for $\lambda_{\text{phys}}$ and $\sigma_{\text{init}}$, as well as practical tuning strategy for $\lambda_{\text{phys}}$ are provided in Appendix~\ref{app:sens}.

\begin{table}[htbp]

\centering
\small
\setlength{\tabcolsep}{3.6pt}
\renewcommand{\arraystretch}{1.05}
\caption{
Ablation of Physics guidance in one-step Phys-Instruct training on Darcy flow.
\textbf{Default} denotes full training configuration.
We report the PDE Error ($\downarrow$). Kimg is the starting point of adding physics guidance.
}\label{tab:ablation_main}
\adjustbox{max width=\textwidth,center}{
\begin{tabular}{l c c c c}
\toprule
Setting & Guidance & kimg & $\lambda_{\mathrm{phys}}$ & PDE Error$\downarrow$  \\
\midrule
Default & \cmark & 0   & $5\mathrm{e}{-3}$ & $\mathbf{3.15{\times}10^{-1}}$     \\
w/o Guidance & \xmark & --  & 0               & $8.75$    \\
Late Guidance& \cmark & 5k & $5\mathrm{e}{-3}$ & $4.07{\times}10^{-1}$    \\
Weak Guidance  & \cmark & 0   & $1\mathrm{e}{-3}$ &$3.59{\times}10^{-1}$  \\
Strong Guidance & \cmark & 0   & $1\mathrm{e}{-2}$ & $3.18{\times}10^{-1}$ \\
\bottomrule
\end{tabular}
}

\end{table}

\subsection{Conditional Downstream Tasks}
\label{sec:downstream}

\begin{table}[b]
\centering
\small
\setlength{\tabcolsep}{3.8pt}
\caption{Forward/inverse results on Darcy flow and Helmholtz at $128\times128$. Prediction and PDE errors are reported ($\downarrow$); Darcy inverse uses absolute error, while others use relative error. Step denotes inference sampling steps.}
\label{tab:darcy_helmholtz_forward_inverse_128}
\adjustbox{center=\textwidth,scale=0.90}{%
\begin{tabular}{l c cc cc cc cc}
\toprule
& & \multicolumn{4}{c}{Darcy Flow} & \multicolumn{4}{c}{Helmholtz} \\
\cmidrule(lr){3-6} \cmidrule(lr){7-10}
Method & Step
& \multicolumn{2}{c}{Forward}
& \multicolumn{2}{c}{Inverse}
& \multicolumn{2}{c}{Forward}
& \multicolumn{2}{c}{Inverse} \\
\cmidrule(lr){3-4} \cmidrule(lr){5-6} \cmidrule(lr){7-8} \cmidrule(lr){9-10}
&
& Rel. Err. & PDE Err.
& Abs. Err. & PDE Err.
& Rel. Err. & PDE Err.
& Rel. Err. & PDE Err. \\
\midrule
Phys-Instruct & 1
& \textbf{0.50\%} & \textbf{1.08}
& \textbf{0.35\%} & \textbf{1.54}
& \textbf{1.86\%} & \textbf{0.42}
& \textbf{4.58\%} & \textbf{0.31} \\
DiffusionPDE  & 2000
& 2.52\% & 2.80
& 9.92\% & 3.11
& 9.69\% & 1.48
& 17.87\% & 0.54 \\
CoCoGen       & 2000
& 2.55\% & 2.76
& 10.11\% & 3.05
& 9.70\% & 1.07
& 17.86\% & 0.47 \\
PIDM          & 2000
& 4.39\% & 3.72
& 13.46\% & 3.97
& 13.76\% & 1.36
& 22.65\% & 0.54 \\
\bottomrule
\end{tabular}%
}
\end{table}

We further show that the distilled students provide strong \emph{unconditional priors} for downstream PDE tasks.
Specifically, we freeze the unconditional Phys-Instruct as a backbone prior and attach a ControlNet-style conditional branch~\citep{zhang2023adding} to incorporate paired task data.
Due to space constraints, details of the conditional pipeline implementation and baselines are deferred to Appendix~\ref{apx:conditional_physinstruct}.
We showcase the forward and inverse problem on Darcy flow and Helmholtz with resolution $128\times128$. Extra downstream experiments can be found in Appendix~\ref{app:extra}.

Table~\ref{tab:darcy_helmholtz_forward_inverse_128} shows that one-step Phys-Instruct consistently outperforms diffusion baselines using
2000 steps on both Darcy flow and Helmholtz. On Darcy flow, it reduces forward prediction error
by about $5$--$9\times$ and inverse error by about $28$--$38\times$, while also lowering PDE Error.
On Helmholtz, it achieves about $5$--$7\times$ lower forward error and $4$--$5\times$ lower inverse error, again with improved PDE consistency.
Figure~\ref{fig:helm_fwd} shows qualitative comparison of Phys-Instruct with other baseline methods on Helmholtz equation forward problem. It demonstrates that Phys-Instruct achieves low relative error and PDE error with details well captured.
These results show that the distilled prior transfers
effectively to conditional PDE tasks at orders-of-magnitude lower sampling cost.
\begin{figure}
    \centering
    \includegraphics[width=0.9\linewidth]{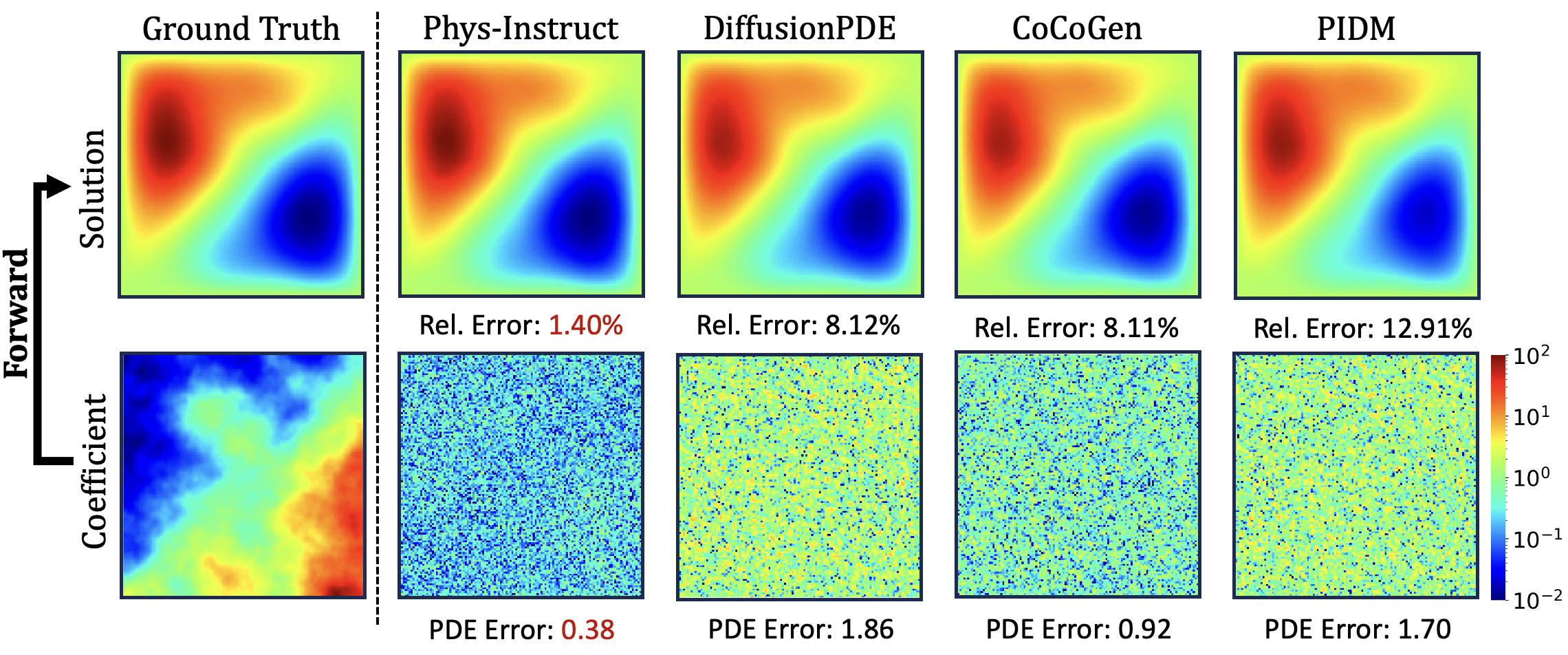}
\caption{
Qualitative comparison for the Helmholtz forward problem (corresponding to Table~\ref{tab:darcy_helmholtz_forward_inverse_128}).
The leftmost column shows the ground-truth fields, and the remaining columns show predictions from different methods.
The top row visualizes the predicted solution $u$, and the bottom row shows the pointwise absolute PDE error $\lvert r(\mathbf{x})_i\rvert$ over the grid.
We annotate each method with the solution relative error and the overall PDE Error$\downarrow$, highlighting that Phys-Instruct achieves the lowest PDE error and best agreement with the ground truth.}
    \label{fig:helm_fwd}
\end{figure}

\section{Conclusion and Future Work}
We propose Phys-Instruct, a physics-guided distillation framework that compresses a diffusion teacher into a few-step generator for PDE solution generation.
Phys-Instruct augments IKL-based distillation objective with physics constraints, yielding students that produce physically consistent samples with dramatically fewer sampling steps, requiring no test-time physics correction.
Empirically, the few-step Phys-Instruct achieves substantially lower PDE Error than diffusion baselines that require hundreds of sampling steps, and the distilled one-step model serves as a strong unconditional backbone for downstream conditional tasks.

There are several promising directions for future work.
First, it would be valuable to combine Phys-Instruct with more advanced distillation techniques, such as SIM~\citep{Luo2024SIM} and Uni-Instruct~\citep{wang2025uni}, to better transfer the knowledge from the teacher while further improving the diversity of the student-generated samples.
For conditional tasks, adapters~\citep{mou2024t2i} or LoRA~\citep{hu2022lora} may provide parameter-efficient alternatives to the current ControlNet setup. Currently Phys-Instruct is limited to fixed resolution setting, extending the framework to and
resolution-robust or resolution-free settings~\citep{yao2025fundps,
wang2025fundiff} is another important direction.
We hope Phys-Instruct will provide a foundation for building efficient physics-aware generative models.

\section*{Acknowledgements and Disclosures}
Computational experiments were performed in part using the Gautschi
community cluster operated by Purdue IT's Rosen Center for Advanced
Computing (RCAC), and the Majda GPU server in the Department of
Mathematics at Purdue University.

Guang Lin gratefully acknowledges the support from DARPA MATHBAC program under grant number: HR0011-26-3-E152, the National Science Foundation (DMS-2533878, DMS-2053746, DMS-2134209, ECCS-2328241, CBET-2347401 and OAC-2311848), and U.S.~Department of Energy (DOE) Office of Science Advanced Scientific Computing Research program under the "Uncertainty Quantification for Multifidelity Operator Learning (MOLUcQ)" project (Project No. 81739), DE-SC0023161, the SciDAC LEADS Institute, and DOE–Fusion Energy Science, under grant number: DE-SC0024583.

The authors declare no competing interests.

\bibliography{ref}

\clearpage
\appendix
\setcounter{tocdepth}{2}

{
\small
\setlength{\parskip}{0pt}
\setlength{\itemsep}{0pt}
\renewcommand{\baselinestretch}{1.00}\selectfont
\tableofcontents
}

\allowdisplaybreaks
\clearpage
\section{PDE Benchmark Details}
\label{apx:pde}

This section provides benchmark details for the five PDEs used in our experiments.
A PDE sample is denoted by $\mathbf{x}=(\mathbf{a},\mathbf{u})$,
where $\mathbf{a}$ collects coefficient/forcing/initial-condition fields and $\mathbf{u}$
is the corresponding solution field or the target solution state for dynamic problems.

For each benchmark, all fields are represented on a fixed grid
$\mathcal{D}_h=\{\xi_i\}_{i=1}^{N}\subset\Omega$.
For static problems on $\Omega_x\subset\mathbb{R}^d$, we take $\xi=\mathbf{c}\in\Omega_x$ and
$\mathcal{D}_h=\{\mathbf{c}_i\}_{i=1}^{N_x}$.
For dynamic problems on $\Omega=\Omega_x\times\Omega_t$, we use the space--time coordinate
$\xi=(\mathbf{c},\tau)^\top$ and a tensor grid
$\mathcal{D}_h=\{(\mathbf{c}_i,\tau_j)\}_{i=1,j=1}^{N_x,N_t}$.
Following \citet{huang2024diffusionpde}, we generate $50{,}000$ samples as training data and $20{,}000$ as test data per PDE.
Darcy flow coefficient fields are drawn from Gaussian random fields (GRFs), and solutions are computed by solving the resulting sparse grid-discretized elliptic system.
Following the FNO benchmark \cite{li2021fourier}, we generate the periodic Navier--Stokes
data using a Fourier pseudo-spectral solver in the stream-function formulation.
Poisson and inhomogeneous Helmholtz data are generated using second-order finite-difference schemes, and Burgers' data follow \citet{zheng2025sindy, zheng2025muti} with a spectral solver to produce space–time fields.
All fields are stored as discretized arrays on the target grids.

On the same grid $\mathcal{D}_h$, we follow Section~\ref{sec 4.2}
to compute the benchmark physics penalty:
$r(\mathbf{x})_i := \mathcal{G}_h[\mathbf{x}](\xi_i)$ and $\mathcal{R}(\mathbf{x})$ is given by~\eqref{eq:phys_loss}.
When boundary values are prescribed, we enforce them explicitly and evaluate $\mathcal{R}(\mathbf{x})$
only on interior (unconstrained) nodes to avoid boundary-discretization artifacts.
The explicit form of $\mathcal{G}_h$ (and thus $\mathcal{R}(\mathbf{x})$) is benchmark-dependent and will be specified
for each PDE below.

The PDE datasets are used to learn the data distribution $p_0$ by training the multi-step diffusion
\emph{teacher} score model $s_\nu$, which is then frozen, and to construct evaluation/conditional test instances.
In contrast, student distillation does not require paired ground-truth supervision:
during distillation, clean samples are generated on-the-fly as $\mathbf{x}_0=g_\theta(\mathbf{z};k)$,
and physics guidance is applied by evaluating the PDE error penalty $\mathcal{R}(\mathbf{x}_0)$ on these samples.

\subsection{Darcy Flow}

We consider the static Darcy flow on $\Omega_x=(0,1)^2$ with homogeneous Dirichlet boundary:
\begin{equation}
\left\{
\begin{aligned}
-\nabla\!\cdot\!\big(a(\mathbf{c})\,\nabla u(\mathbf{c})\big) &= q(\mathbf{c}),
&& \mathbf{c}\in\Omega_x,\\
u(\mathbf{c}) &= 0,
&& \mathbf{c}\in\partial\Omega_x,
\end{aligned}
\right.
\end{equation}
where the forcing is fixed as $q(\mathbf{c})\equiv 1$.

Each dataset sample is obtained by first drawing a random permeability field and then solving the PDE.
Following DiffusionPDE~\citet{huang2024diffusionpde}, we construct heterogeneous coefficients from a Gaussian
random field sampled on the target grid. Specifically, we draw a zero-mean GRF on a $32\times 32$ or $128\times 128$ grid
using a Gaussian-correlated spectral sampler, and define a binary-valued permeability by pointwise thresholding:
\[
\mathbf{a}_{ij}=
\begin{cases}
12, & \mathbf{m}_{ij}\ge 0,\\
3,  & \mathbf{m}_{ij}< 0,
\end{cases}
\qquad (i,j)\in\mathcal{D}_h,
\]
where $\mathbf{m}\in \mathbb{R}^{32\times 32}$ or $\mathbb{R}^{128\times 128}$  denotes the sampled GRF and $\mathcal{D}_h$ is the uniform grid over $\Omega_x$.
Given $\mathbf{a}$, we solve the corresponding discretized Darcy system with homogeneous Dirichlet boundary conditions,
and store the resulting fields on the same grid.

We denote the discretized fields by boldface symbols:
$\mathbf{a}$ and $\mathbf{u}$ are the grid-sampled
permeability and solution fields, respectively. Thus, each sample is
$\mathbf{x}=(\mathbf{a},\mathbf{u})$, where $\mathbf{a}$ specifies the PDE instance and $\mathbf{u}$ is its solution.

For physics evaluation and guidance, we match the residual discretization to the benchmark solver.
In particular, for Darcy flow we use a conservative flux-form discretization with harmonic face averaging,
rather than applying a naive central-difference Laplacian to $a\nabla u$.
Let $S=32$ or $128$ and $h=1/(S-1)$. For interior grid nodes $(i,j)\in\mathcal{I}_{\mathrm{int}}$, we define
\begin{equation}
\begin{aligned}
r_{\mathrm{Darcy}}(\mathbf{x})_{i,j}
:=
-&\frac{1}{h^2}\Big[
\mathbf{a}_{i+\frac12,j}(\mathbf{u}_{i+1,j}-\mathbf{u}_{i,j})
-\mathbf{a}_{i-\frac12,j}(\mathbf{u}_{i,j}-\mathbf{u}_{i-1,j}) \\
&\hspace{0.55in}
+\mathbf{a}_{i,j+\frac12}(\mathbf{u}_{i,j+1}-\mathbf{u}_{i,j})
-\mathbf{a}_{i,j-\frac12}(\mathbf{u}_{i,j}-\mathbf{u}_{i,j-1})
\Big]
-q_{i,j},
\end{aligned}
\label{eq:darcy_residual_stencil}
\end{equation}
where the face coefficients are computed by harmonic averaging:
\begin{equation}
\mathbf{a}_{i+\frac12,j}
=
\frac{2\mathbf{a}_{i,j}\mathbf{a}_{i+1,j}}
{\mathbf{a}_{i,j}+\mathbf{a}_{i+1,j}},
\qquad
\mathbf{a}_{i,j+\frac12}
=
\frac{2\mathbf{a}_{i,j}\mathbf{a}_{i,j+1}}
{\mathbf{a}_{i,j}+\mathbf{a}_{i,j+1}},
\end{equation}
with analogous definitions for $\mathbf{a}_{i-\frac12,j}$ and $\mathbf{a}_{i,j-\frac12}$.
The homogeneous Dirichlet boundary values are imposed explicitly, and the PDE error is evaluated only on interior nodes
$\mathcal{I}_{\mathrm{int}}$ to avoid boundary-discretization artifacts.

\subsection{Poisson Equation}

We consider the static Poisson equation on $\Omega_x=(0,1)^2$ with homogeneous Dirichlet boundary:
\begin{equation}
\left\{
\begin{aligned}
\nabla^{2}u(\mathbf{c}) &= a(\mathbf{c}),
&& \mathbf{c}\in\Omega_x,\\
u(\mathbf{c}) &= 0,
&& \mathbf{c}\in\partial\Omega_x.
\end{aligned}
\right.
\end{equation}

Each dataset sample is generated by first drawing a random forcing field and then solving the PDE.
Following ~\citet{huang2024diffusionpde}, we set $a=\mu$ where
$\mu$ is sampled from the Gaussian random field.
We solve for $u$ using a second-order finite-difference scheme, enforcing the homogeneous Dirichlet boundary
via the standard mollifier $\sin(\pi c_1)\sin(\pi c_2)$ in the data-generation pipeline.
We store the discretized fields on a $32\times 32$ grid and denote each sample by
$\mathbf{x}=(\mathbf{a},\mathbf{u})$.

On $\mathcal{D}_h\subset\Omega_x$, we evaluate the benchmark-consistent discrete PDE error map
$\mathcal{G}_h$ using the second-order Laplacian, yielding the interior PDE error
\begin{equation}
r_{\mathrm{Pois}}(\mathbf{x})_i := (\Delta_h \mathbf{u})_i - \mathbf{a}_i,
\qquad i\in \mathcal{I}_{\mathrm{int}},
\end{equation}
where $\Delta_h$ is the standard second-order discrete Laplacian on the uniform grid.

\subsection{Inhomogeneous Helmholtz Equation}

We consider the static inhomogeneous Helmholtz equation on $\Omega_x=(0,1)^2$ with homogeneous Dirichlet boundary:
\begin{equation}
\left\{
\begin{aligned}
\nabla^{2}u(\mathbf{c}) + k^{2}u(\mathbf{c}) &= a(\mathbf{c}),
&& \mathbf{c}\in\Omega_x,\\
u(\mathbf{c}) &= 0,
&& \mathbf{c}\in\partial\Omega_x,
\end{aligned}
\right.
\end{equation}
where we set $k=1$ (note that this reduces to the Poisson equation when $k=0$).

This benchmark follows the same generation pipeline as Poisson and we store each sample $\mathbf{x}=(\mathbf{a},\mathbf{u})$ on a $128\times 128$ grid.

Accordingly, the discrete PDE error differs from Poisson only by the additional reaction term:
\begin{equation}
r_{\mathrm{Helm}}(\mathbf{x})_i := (\Delta_h \mathbf{u})_i + k^2 \mathbf{u}_i - \mathbf{a}_i,
\qquad i\in \mathcal{I}_{\mathrm{int}},
\end{equation}
with the same $\Delta_h$ as above.

\subsection{Non-bounded Navier--Stokes Equation}

We study the 2D incompressible Navier--Stokes benchmark on $\Omega_x=(0,1)^2$
with periodic (non-bounded) boundary conditions in the vorticity formulation:
\begin{equation}
\left\{
\begin{aligned}
\partial_{\tau}\omega(\mathbf{c},\tau)
+ \mathbf{v}(\mathbf{c},\tau)\cdot\nabla \omega(\mathbf{c},\tau)
&= \nu\,\Delta\omega(\mathbf{c},\tau) + q(\mathbf{c}),
&& (\mathbf{c},\tau)\in\Omega_x\times(0,T],\\
\nabla\cdot \mathbf{v}(\mathbf{c},\tau) &= 0,
&& (\mathbf{c},\tau)\in\Omega_x\times(0,T],
\end{aligned}
\right.
\end{equation}
where $\nu=10^{-3}$, and $\omega=\nabla\times\mathbf{v}$ is the scalar vorticity.
The forcing follows the fixed pattern
\begin{equation}
q(\mathbf{c})=\frac{1}{10}\Big(\sin\big(2\pi(c_1+c_2)\big)+\cos\big(2\pi(c_1+c_2)\big)\Big).
\end{equation}

Following ~\citet{huang2024diffusionpde}, the initial condition is sampled from a Gaussian random field
\begin{equation}
\omega_0 \sim \mathcal{N}\!\big(0,\ 7^{1.5}(-\Delta+49\mathbf I)^{-2.5}\big).
\end{equation}
Given $\omega_0$ and the fixed forcing $q$, we solve the dynamics using a pseudo-spectral method in the
stream-function formulation: we advance the vorticity equation in the Fourier domain and evaluate nonlinear terms
in physical space via inverse FFTs. We simulate for $1$ second using $10$ time steps.
In the dataset convention of DiffusionPDE, we take the terminal snapshot index as $T=10$.
We store $(\omega_0,\omega_T)$ on a $64\times 64$ spatial grid and denote the resulting discrete sample by
$\mathbf{x}=(\boldsymbol{\omega}_0,\boldsymbol{\omega}_T)$, i.e., $\mathbf{a}=\boldsymbol{\omega}_0$ and
$\mathbf{u}=\boldsymbol{\omega}_T$, yielding resolution $64\times 64\times 2$.

As in ~\citet{huang2024diffusionpde}, since $T$ is large in the benchmark convention, we do not evaluate the full time-dependent
PDE error at the terminal state. Instead, we adopt a simplified terminal-time constraint evaluated directly on the predicted vorticity snapshot as terminal vorticity-gradient surrogate and define
\begin{equation}
r_{\mathrm{NS}}(\mathbf{x})_{i,j}
:= (\partial_{x,h}\boldsymbol{\omega}_T)_{i,j} + (\partial_{y,h}\boldsymbol{\omega}_T)_{i,j},
\qquad (i,j)\in\mathcal{I}_{\mathrm{all}},
\end{equation}
where $\partial_{x,h}$ and $\partial_{y,h}$ are central differences.

\subsection{Burgers' Equation}

We consider the 1D viscous Burgers' equation on $\Omega_x=(0,1)$ with periodic boundary conditions:
\begin{equation}
\left\{
\begin{aligned}
\partial_{\tau}u(c,\tau) + \partial_{c}\!\left(\frac{u(c,\tau)^{2}}{2}\right)
- \nu\,\partial_{cc}u(c,\tau) &= 0,
&& (c,\tau)\in\Omega_x\times(0,T],\\
u(c,0) &= u_{0}(c),
&& c\in\Omega_x,
\end{aligned}
\right.
\end{equation}
where $\nu=0.01$, we take the physical simulation horizon $T=1$, and enforce periodicity in space
(e.g., $u(0,\tau)=u(1,\tau)$ for all $\tau\in(0,T]$).

Following ~\citet{huang2024diffusionpde} and the FNO benchmark procedure, the initial condition is sampled
as a Gaussian random field
\begin{equation}
u_0 \sim \mathcal{N}\!\big(0,\ 625(-\Delta+25\mathbf I)^{-2}\big).
\end{equation}
Given $u_0$, we solve the PDE using a spectral method and simulate over $t\in[0,1]$ with $128$ uniformly spaced time snapshots (including $t=0$), producing a space--time solution field $u_{0:T}$.
We store $\mathbf{u}=u_{0:T}$ on a $128\times 128$ tensor grid over $(c,\tau)$.

Since the time dimension is densely discretized, we compute the pointwise PDE error by finite differences:
\begin{equation}
r_{\mathrm{Burg}}(\mathbf{x})_{i,j}
:= (\partial_{\tau,h}\mathbf{u})_{i,j}
+ \big(\partial_{c,h}(\tfrac12 \mathbf{u}^2)\big)_{i,j}
- \nu\,(\partial_{cc,h}\mathbf{u})_{i,j},
\qquad (i,j)\in\mathcal{I}_{\mathrm{all}},
\end{equation}
where $\partial_{\tau,h}$, $\partial_{c,h}$ and $\partial_{cc,h}$ are standard discrete operators with periodic
wrapping in space.

\section{Proof of Theorem ~\ref{thm:phys_impl_one}}
\label{apx:proof}

\begin{proof}[Proof of Theorem~\ref{thm:phys_impl_one}]
We first justify the penalized objective used in the analysis.
The constrained formulation \eqref{eq:constrained_ikl} is not optimized directly in our implementation.
Instead, we adopt a penalty relaxation and treat $\lambda_{\mathrm{phys}}>0$ as a trade-off hyperparameter,
which augments the IKL matching term by a soft penalty on the constraint violation.
Concretely, consider the Lagrangian-type surrogate
\[
\mathcal{D}^{[0,T]}_{\mathrm{IKL}}(q_\theta\|p)
+\lambda_{\mathrm{phys}}
\Big(\mathbb{E}_{\mathbf{x}_0\sim q_\theta}\big[\mathcal{R}(\mathbf{x}_0)\big]-\varepsilon\Big).
\]
Since $-\lambda_{\mathrm{phys}}\varepsilon$ is a constant independent of $\theta$,
it does not affect the gradient with respect to $\theta$ and can be dropped.
Moreover, under our unified $k$-step generator parameterization,
a clean student sample satisfies $\mathbf{x}_0=g_\theta(\mathbf{z};k)$ with
$\mathbf{z}\sim p_{\mathbf{z}}$ and $k\sim\Delta_K$, which induces the implicit distribution $q_\theta$.
Therefore,
\[
\mathbb{E}_{\mathbf{x}_0\sim q_\theta}\big[\mathcal{R}(\mathbf{x}_0)\big]
=
\mathbb{E}_{\mathbf{z}\sim p_z,\,k\sim \Delta_K,\mathbf{x}_0=g_\theta(\mathbf{z};k)}\Big[\mathcal{R}\big(\mathbf{x}_0\big)\Big].
\]
Combining the above yields the implemented penalized objective
\[
\mathcal{L}_{G}(\theta)
=
\mathcal{D}^{[0,T]}_{\mathrm{IKL}}(q_\theta\|p)
+\lambda_{\mathrm{phys}}\,
\mathbb{E}_{\mathbf{z}\sim p_z,\,k\sim \Delta_K,\mathbf{x}_0=g_\theta(\mathbf{z};k)}\Big[\mathcal{R}\big(\mathbf{x}_0\big)\Big],
\]
which is exactly the objective optimized by the guided generator update.

Throughout, we assume mild regularity so that differentiation can be interchanged with the time integral
and expectations, and the sampling procedure that generates $\mathbf{x}_t\sim p_{0t}(\cdot\mid \mathbf{x}_0)$ admits a
pathwise derivative with respect to $\mathbf{x}_0$ (and thus with respect to $\theta$ through $\mathbf{x}_0=g_\theta(\mathbf{z};k)$).
We also work under the ideal score limit in the theorem, i.e.,
$s_\phi(\mathbf{x}_t,t)=\nabla_{\mathbf{x}_t}\log q_{\theta,t}(\mathbf{x}_t)$ and $s_\nu(\mathbf{x}_t,t)=\nabla_{\mathbf{x}_t}\log p_t(\mathbf{x}_t)$
for a.e.\ $t\in[0,T]$ and $\mathbf{x}_t\sim q_{\theta,t}$.

Fix $t\in[0,T]$. The time-$t$ KL term is
\[
\mathcal{D}_{\mathrm{KL}}(q_{\theta,t}\|p_t)
=
\mathbb{E}_{\mathbf{x}_t\sim q_{\theta,t}}\!\big[\log q_{\theta,t}(\mathbf{x}_t)-\log p_t(\mathbf{x}_t)\big].
\]
By the sampling procedure in~\eqref{eq:physinstruct_loss}, $\mathbf{x}_t$ is obtained by drawing
$k\sim\Delta_K$, $\mathbf{z}\sim p_\mathbf{z}$, setting $\mathbf{x}_0=g_\theta(\mathbf{z};k)$, and sampling $\mathbf{x}_t\sim p_{0t}(\cdot\mid \mathbf{x}_0)$.
Under the above regularity, differentiating the KL term along this sampling path yields
\begin{equation}\label{eq:kl_pathwise_pf_new}
\nabla_\theta \mathcal{D}_{\mathrm{KL}}(q_{\theta,t}\|p_t)
=
\mathbb{E}\Big[
\nabla_\theta \log q_{\theta,t}(\mathbf{x}_t)
+
\big(\nabla_{\mathbf{x}_t}\log q_{\theta,t}(\mathbf{x}_t)-\nabla_{\mathbf{x}_t}\log p_t(\mathbf{x}_t)\big)^\top \nabla_\theta \mathbf{x}_t
\Big],
\end{equation}
where the expectation is over $k\sim\Delta_K$, $\mathbf{z}\sim p_\mathbf{z}$, $\mathbf{x}_0=g_\theta(\mathbf{z};k)$, and
$\mathbf{x}_t\sim p_{0t}(\cdot\mid \mathbf{x}_0)$. The first term vanishes by the score identity
$\mathbb{E}_{\mathbf{x}_t\sim q_{\theta,t}}[\nabla_\theta\log q_{\theta,t}(\mathbf{x}_t)]=0$. Therefore,
\begin{equation}\label{eq:kl_pathwise_pf_final_new}
\nabla_\theta \mathcal{D}_{\mathrm{KL}}(q_{\theta,t}\|p_t)
=
\mathbb{E}\Big[
\big(\nabla_{\mathbf{x}_t}\log q_{\theta,t}(\mathbf{x}_t)-\nabla_{\mathbf{x}_t}\log p_t(\mathbf{x}_t)\big)^\top \nabla_\theta \mathbf{x}_t
\Big].
\end{equation}

Recalling
$\mathcal{D}^{[0,T]}_{\mathrm{IKL}}(q_\theta\|p)
=\int_0^T w(t)\,\mathcal{D}_{\mathrm{KL}}(q_{\theta,t}\|p_t)\,\mathrm{d}t$
and interchanging $\nabla_\theta$ with the integral, we obtain
\begin{equation}\label{eq:grad_IKL_pathwise_pf_new}
\nabla_\theta \mathcal{D}^{[0,T]}_{\mathrm{IKL}}(q_\theta\|p)
=
\int_0^T w(t)\,
\mathbb{E}\Big[
\big(\nabla_{\mathbf{x}_t}\log q_{\theta,t}(\mathbf{x}_t)-\nabla_{\mathbf{x}_t}\log p_t(\mathbf{x}_t)\big)^\top \nabla_\theta \mathbf{x}_t
\Big]\mathrm{d}t.
\end{equation}
Under the ideal score limit, $\nabla_{\mathbf{x}_t}\log q_{\theta,t}(\mathbf{x}_t)=s_\phi(\mathbf{x}_t,t)$ and
$\nabla_{\mathbf{x}_t}\log p_t(\mathbf{x}_t)=s_\nu(\mathbf{x}_t,t)$, so \eqref{eq:grad_IKL_pathwise_pf_new} becomes
\begin{equation}\label{eq:grad_IKL_scores_pf_new}
\nabla_\theta \mathcal{D}^{[0,T]}_{\mathrm{IKL}}(q_\theta\|p)
=
\int_0^T w(t)\,
\mathbb{E}\Big[
\big(s_\phi(\mathbf{x}_t,t)-s_\nu(\mathbf{x}_t,t)\big)^\top \nabla_\theta \mathbf{x}_t
\Big]\mathrm{d}t.
\end{equation}
In the implemented update, the score term is treated as a constant weight with respect to $\theta$,
i.e., gradients are stopped through the score evaluations. Equivalently, we may insert the
stop-gradient operator $\mathrm{SG}\{\cdot\}$ on the score difference inside the expectation in
\eqref{eq:grad_IKL_scores_pf_new} without changing the resulting $\theta$-gradient, since the only
$\theta$-dependence being differentiated is carried by the pathwise term $\nabla_\theta x_t$.
Therefore,
\begin{equation}\label{eq:grad_match_DI_pf_new}
\nabla_\theta \mathcal{D}^{[0,T]}_{\mathrm{IKL}}(q_\theta\|p)
=
\int_{0}^{T} w(t)\,
\mathbb{E}\Big[
\mathrm{SG}\!\big\{ s_\phi(\mathbf{x}_t,t)-s_\nu(\mathbf{x}_t,t)\big\}^{\!\top}\nabla_\theta \mathbf{x}_t
\Big]\,dt,
\end{equation}
where the expectation is over $\mathbf{z}\sim p_\mathbf{z}$, $k\sim\Delta_K$, $\mathbf{x}_0=g_\theta(\mathbf{z};k)$, and
$\mathbf{x}_t\sim p_{0t}(\cdot\mid \mathbf{x}_0)$.

For the physics term, recall from~\eqref{eq:physinstruct_loss} that the penalty enters as
$\lambda_{\mathrm{phys}}\mathcal{L}_{\mathrm{phys}}(\theta)$, where
\[
\mathcal{L}_{\mathrm{phys}}(\theta)
=
\mathbb{E}_{\mathbf{z}\sim p_{\mathbf{z}},\,k\sim\Delta_K,\ \mathbf{x}_0=g_\theta(\mathbf{z};k)}
\big[\mathcal{R}(\mathbf{x}_0)\big].
\]
Under the same regularity assumptions, the chain rule yields
\begin{equation}\label{eq:grad_phys_pf_new}
\nabla_\theta \mathcal{L}_{\mathrm{phys}}(\theta)
=
\mathbb{E}_{\mathbf{z}\sim p_{\mathbf{z}},\,k\sim\Delta_K}
\Big[
\big(\nabla_{\mathbf{x}_0}\mathcal{R}(\mathbf{x}_0)\big)^{\!\top}\,
\nabla_\theta g_\theta(\mathbf{z};k)
\Big]_{\mathbf{x}_0=g_\theta(\mathbf{z};k)}.
\end{equation}
Combining~\eqref{eq:grad_match_DI_pf_new} with~\eqref{eq:grad_phys_pf_new} gives
\begin{equation*}
    \begin{aligned}
        \nabla_\theta \mathcal{L}_{G}(\theta)
=
&\int_{0}^{T} w(t)\,
\mathbb{E}_{\substack{\mathbf{z}\sim p_{\mathbf{z}},\,k\sim\Delta_K\\
\mathbf{x}_0=g_\theta(\mathbf{z};k),\,\mathbf{x}_t\sim p_{0t}(\cdot\mid \mathbf{x}_0)}}\Big[
\mathrm{SG}\!\big\{ s_\phi(\mathbf{x}_t,t)-s_\nu(\mathbf{x}_t,t)\big\}^{\!\top}\nabla_\theta \mathbf{x}_t
\Big]\mathrm{d}t
\;\\&+\;
\lambda_{\mathrm{phys}}
\mathbb{E}_{\mathbf{z}\sim p_{\mathbf{z}},\,k\sim\Delta_K}
\Big[
\big(\nabla_{\mathbf{x}_0}\mathcal{R}(\mathbf{x}_0)\big)^{\!\top}\,
\nabla_\theta g_\theta(\mathbf{z};k)
\Big]_{\mathbf{x}_0=g_\theta(\mathbf{z};k)},
    \end{aligned}
\end{equation*}
which is exactly~\eqref{eq:phys_main_one}. This completes the proof.
\end{proof}

\section{Phys-Instruct Algorithm}
Algorithm~\ref{alg:physinstruct} summarizes the training procedure of Phys-Instruct. The pretrained teacher score model is kept frozen, while the auxiliary score model and the student generator are updated alternately; the physics guidance is applied only to the generator update.

\begin{algorithm}[t]
\caption{Phys-Instruct.}
\label{alg:physinstruct}
\begin{algorithmic}[1]
\STATE \textbf{Input:} pre-trained model $s_\nu$; generator $g_\theta$; prior $p_z$; auxiliary $s_\phi$; regularization weight $\lambda_{\mathrm{phys}}$; max step $K$.
\WHILE{not converge}
    \STATE \textit{// Score-learning phase for auxiliary diffusion model}
    \STATE Sample $\mathbf{z} \sim p_z$, $k \sim \Delta_K$, set $\mathbf{x}_0 = g_\theta(\mathbf{z};k)$ and diffuse to $\mathbf{x}_t$ via the forward SDE \eqref{eq:fwd-process}.
    \STATE Update $\phi$ using SGD with loss function \eqref{eq:dsm}. \vspace{0.05 in}
    \STATE \textit{// Instruction + physics phase for generator}
    \STATE Sample $\mathbf{z} \sim p_z$, $k \sim \Delta_K$, set $\mathbf{x}_0 = g_\theta(\mathbf{z};k)$ and diffuse to $\mathbf{x}_t$ via the forward SDE \eqref{eq:fwd-process}.
    \STATE Update $\theta$ using SGD with loss function \eqref{eq:physinstruct_loss}, whose tractable gradient is \eqref{eq:phys_main_one}.
\ENDWHILE
\STATE \textbf{Output:} $\theta, \phi$.
\end{algorithmic}
\end{algorithm}

\section{Distributional Metrics}
\label{app:dist_metrics}

In addition to PDE Error, we report distributional discrepancy between generated samples and held-out real data using two-sample distances.
These metrics quantify how closely the generated fields match the data distribution, complementing physics-consistency evaluations.

\paragraph{Preprocessing and Normalization.}
All distances are computed on the same field representation used for evaluation.
For multi-field problems (e.g., coefficient--solution pairs), we apply channel-wise standardization using statistics computed from the real data:
for each channel $c$, we transform both real and generated samples by
\begin{equation}
\tilde{\mathbf{x}}^{(c)} = \frac{\mathbf{x}^{(c)} - \mu^{(c)}_{\text{real}}}{\sigma^{(c)}_{\text{real}} + \epsilon},
\end{equation}
where $\mu^{(c)}_{\text{real}}$ and $\sigma^{(c)}_{\text{real}}$ are the mean and standard deviation over the real dataset, and $\epsilon$ is a small constant.
For single-field benchmarks such as Burgers' equation, this reduces to standardization of the scalar field.

\paragraph{Sliced Wasserstein Distance (SWD) \citep{wu2019sliced,kolouri2019generalized}.}
Let $P$ and $Q$ denote the empirical distributions of standardized generated and real samples, respectively.
SWD approximates the Wasserstein distance in high dimensions by averaging 1D Wasserstein distances over random projections:
\begin{equation}
\mathrm{SWD}(P,Q)
=
\frac{1}{K}\sum_{k=1}^{K}
W_1\!\left( \langle \mathbf{v}_k, \mathbf{x} \rangle_{\mathbf{x}\sim P},\
            \langle \mathbf{v}_k, \mathbf{x} \rangle_{\mathbf{x}\sim Q}
\right),
\end{equation}
where $\{\mathbf{v}_k\}_{k=1}^{K}$ are i.i.d.\ random unit vectors and $W_1(\cdot,\cdot)$ is the 1D Wasserstein distance computed via sorting.
We use a fixed random seed and $K=512$ random projections, which provides a stable estimate of SWD at a moderate computational cost for high-dimensional fields.

\paragraph{Maximum Mean Discrepancy (MMD) \citep{sutherland2017generative}.}
We also report MMD with a multi-scale RBF kernel as a complementary two-sample statistic:
\begin{equation}
\mathrm{MMD}^2(P,Q)
=
\mathbb{E}_{\mathbf{x},\mathbf{x}'\sim P}[k(\mathbf{x},\mathbf{x}')]
+
\mathbb{E}_{\mathbf{y},\mathbf{y}'\sim Q}[k(\mathbf{y},\mathbf{y}')]
-
2\,\mathbb{E}_{\mathbf{x}\sim P,\mathbf{y}\sim Q}[k(\mathbf{x},\mathbf{y})],
\end{equation}
with $k(\mathbf{x},\mathbf{y}) = \exp\!\left(-\|\mathbf{x}-\mathbf{y}\|_2^2/(2\sigma^2)\right)$.
Bandwidths are selected from the real data using the median pairwise-distance heuristic:
we compute $\hat{\sigma}$ as the median of pairwise distances between a random subset of real samples, and use the multi-scale set
$\sigma \in \{\hat{\sigma}/4,\ \hat{\sigma}/2,\ \hat{\sigma},\ 2\hat{\sigma},\ 4\hat{\sigma}\}$.
We report $\sqrt{\mathrm{MMD}^2}$ for readability.

\subsection{Additional Distributional Coverage and Sample Diversity on Burgers' Equation}
\label{app:burgers_diversity}

To further examine what is preserved by physics-guided distillation, we complement the SWD and MMD results in Figure~\ref{fig:dist_metrics_burgers} with distributional coverage and sample-diversity measurements on Burgers' equation.
We extend the student sampling budget to $k\in\{1,2,4,8\}$ and compare against the 200-step EDM teacher. Here kNN Recall measures support coverage, and Diversity/T200 measures sample-to-sample variation relative to the 200-step teacher, whose value is normalized to $1$.

Table~\ref{tab:burgers_diversity} shows that one-step distillation reduces both support coverage and sample variation, despite substantially improving PDE Error.
Most of the coverage gap is recovered by 4 steps: Recall increases from $0.674$ to $0.930$, compared to $0.980$ for EDM-200, while Diversity/T200 increases from $0.658$ to $0.890$.
At the same time, SWD approaches the teacher value ($0.049$ versus $0.043$), and PDE Error remains substantially lower ($0.055$ versus $0.860$).
Increasing the budget to 8 steps further improves PDE Error and SWD, but yields only a small increase in relative diversity and a slight decrease in Recall. On this benchmark, most of the coverage improvement is obtained by $k=4$, while the additional gains at $k=8$ are smaller.

\begin{table}[htbp]
\centering
\small
\setlength{\tabcolsep}{4pt}
\caption{Additional distributional coverage and sample diversity on Burgers' equation.}
\label{tab:burgers_diversity}
\begin{tabular}{lcccccc}
\toprule
Method & Step & Inference Latency (s)$\downarrow$ & PDE Error$\downarrow$ & SWD$\downarrow$ & Recall$\uparrow$ & Diversity/T200 \\
\midrule
\multirow{4}{*}{Phys-Instruct}
& 1 & 0.012 & 0.171 & 0.071 & 0.674 & 0.658 \\
& 2 & 0.036 & 0.118 & 0.060 & 0.830 & 0.896 \\
& 4 & 0.084 & 0.055 & 0.049 & 0.930 & 0.890 \\
& 8 & 0.225 & 0.046 & 0.046 & 0.919 & 0.905 \\
\midrule
EDM & 200 & 4.919 & 0.860 & 0.043 & 0.980 & 1.000 \\
\bottomrule
\end{tabular}
\end{table}

\section{Normalized Compute and Update-Level Breakdown}
\label{app:compute_cost}

To complement the wall-clock discussion, we report normalized training compute. Since teacher pretraining and Phys-Instruct distillation were conducted on different GPU hardware, raw runtime is not directly comparable across stages. We therefore use FLOPs-based measurements when such profiling is available.

Table~\ref{tab:compute_total} summarizes the total training compute. For Darcy Flow, Poisson, and non-bounded Navier--Stokes, where both teacher and distillation were profiled in our setup, Phys-Instruct requires about $74\%$ of the teacher-training compute. For Burgers' and Helmholtz, the teacher checkpoints were inherited from prior work rather than retrained in-house. We therefore report only the measured distillation compute for these two benchmarks and omit teacher-side totals and D/T ratios.

Table~\ref{tab:compute_breakdown} further decomposes the per-update distillation cost into the auxiliary score-model update ($s_\phi$), the generator update ($g_\theta$), and the PDE-guidance portion within $g_\theta$. Across all benchmarks, the overall cost is dominated by the generator-side update, while the PDE residual contributes a stable overhead of about one-third of the generator-update cost. This indicates that the added physics guidance remains bounded and is not the primary computational bottleneck.

\begin{table}[t]
\centering
\small
\caption{Normalized total training compute in kTFLOPs. D/T denotes the ratio between Phys-Instruct distillation cost and teacher-training cost. For Burgers' and Helmholtz, the teacher checkpoints were inherited from prior work, so comparable teacher-training compute is not reported.}
\label{tab:compute_total}
\begin{tabular}{lccc}
\toprule
Benchmark & Teacher & Distill & D/T \\
\midrule
Darcy Flow & 689.4  & 506.9  & 0.735 \\
Poisson    & 689.4  & 506.9  & 0.735 \\
$\text{NS}_{\text{unb.}}$         & 1609.0 & 1191.1 & 0.740 \\
Burgers'   & ---    & 3185.3 & --- \\
Helmholtz  & ---    & 3364.9 & --- \\
\bottomrule
\end{tabular}
\end{table}

\begin{table}[t]
\centering
\small
\caption{Update-level compute breakdown of Phys-Instruct distillation, reported in TFLOPs and normalized by batch size. Here, PDE denotes the portion of the generator update attributable to PDE guidance, so Total corresponds to $s_\phi + g_\theta$, while PDE/G reports the fraction of $g_\theta$ consumed by PDE guidance.}
\label{tab:compute_breakdown}
\begin{tabular}{lccccc}
\toprule
Benchmark & $s_\phi$ & $g_\theta$ & PDE & PDE/G & Total \\
\midrule
Darcy      & 0.086 & 0.128 & 0.042 & 33.2\% & 0.213 \\
Poisson    & 0.086 & 0.128 & 0.042 & 33.2\% & 0.213 \\
$\text{NS}_{\text{unb.}}$         & 0.269 & 0.440 & 0.141 & 32.0\% & 0.709 \\
Burgers'   & 1.328 & 1.990 & 0.663 & 33.3\% & 3.318 \\
Helmholtz  & 1.328 & 1.991 & 0.663 & 33.3\% & 3.319 \\
\bottomrule
\end{tabular}
\end{table}

\section{Additional Phys-Instruct Generation Results}
\label{app:additional_uncond}
\subsection{Sample-Level Variability}
\label{app:std_results}

To complement the mean PDE Error reported in Table~\ref{tab:uncond}, we additionally compute the sample standard deviation of the per-sample PDE Error for Phys-Instruct in Table~\ref{tab:std_physinstruct}.
The statistics are evaluated using 1000 randomly generated samples for each benchmark and sampling budget. Across all five benchmarks, increasing the sampling budget from 1 to 4 steps reduces the mean PDE Error.
The sample standard deviation also generally decreases with more sampling steps, indicating that the improvement is not only reflected in the average error but also in more concentrated sample-level PDE errors.

\begin{table}[htbp]
\centering
\small
\setlength{\tabcolsep}{5pt}
\caption{Sample-level variability of Phys-Instruct PDE Error over 1000 generated samples. We report mean $\pm$ standard deviation.}
\label{tab:std_physinstruct}
\begin{tabular}{l c c c}
\toprule
Benchmark & Resolution & Step & PDE Error \\
\midrule
\multirow{4}{*}{Darcy Flow}
& \multirow{4}{*}{$32\times32\times2$}
& 1 & $3.15{\times}10^{-1} \pm 6.81{\times}10^{-2}$ \\
& & 2 & $2.18{\times}10^{-1} \pm 5.27{\times}10^{-2}$ \\
& & 3 & $1.35{\times}10^{-1} \pm 3.83{\times}10^{-2}$ \\
& & 4 & $8.49{\times}10^{-2} \pm 2.13{\times}10^{-2}$ \\
\midrule
\multirow{4}{*}{Poisson}
& \multirow{4}{*}{$32\times32\times2$}
& 1 & $2.16{\times}10^{-2} \pm 5.57{\times}10^{-3}$ \\
& & 2 & $5.18{\times}10^{-3} \pm 4.89{\times}10^{-3}$ \\
& & 3 & $4.79{\times}10^{-3} \pm 3.86{\times}10^{-3}$ \\
& & 4 & $3.99{\times}10^{-3} \pm 3.55{\times}10^{-3}$ \\
\midrule
\multirow{4}{*}{NS$_{\mathrm{unb.}}$}
& \multirow{4}{*}{$64\times64\times2$}
& 1 & $3.15{\times}10^{-2} \pm 6.74{\times}10^{-3}$ \\
& & 2 & $1.46{\times}10^{-2} \pm 5.19{\times}10^{-3}$ \\
& & 3 & $3.65{\times}10^{-3} \pm 4.78{\times}10^{-3}$ \\
& & 4 & $3.21{\times}10^{-3} \pm 4.30{\times}10^{-3}$ \\
\midrule
\multirow{4}{*}{Burgers'}
& \multirow{4}{*}{$128\times128\times1$}
& 1 & $1.71{\times}10^{-1} \pm 3.86{\times}10^{-2}$ \\
& & 2 & $1.18{\times}10^{-1} \pm 3.05{\times}10^{-2}$ \\
& & 3 & $7.17{\times}10^{-2} \pm 2.74{\times}10^{-2}$ \\
& & 4 & $5.54{\times}10^{-2} \pm 2.20{\times}10^{-2}$ \\
\midrule
\multirow{4}{*}{Helmholtz}
& \multirow{4}{*}{$128\times128\times2$}
& 1 & $3.75{\times}10^{-1} \pm 1.14{\times}10^{-1}$ \\
& & 2 & $3.00{\times}10^{-1} \pm 1.12{\times}10^{-1}$ \\
& & 3 & $2.10{\times}10^{-1} \pm 8.66{\times}10^{-2}$ \\
& & 4 & $1.46{\times}10^{-1} \pm 6.84{\times}10^{-2}$ \\
\bottomrule
\end{tabular}
\end{table}

\subsection{Resolution Scalability on Poisson}
\label{app:poisson_resolution_scaling}

To further examine scalability with respect to spatial resolution, we additionally evaluate one-step Phys-Instruct on the Poisson benchmark at resolutions $128 \times 128 \times 2$, $256 \times 256 \times 2$, and $512 \times 512 \times 2$, complementing the $32 \times 32 \times 2$ result reported in the main submission. All experiments use the same data-generation and evaluation pipeline.

Table~\ref{tab:poisson_resolution_scaling} summarizes the PDE Error, the ratio of student PDE Error to teacher PDE Error, the per-sample inference time, and the peak GPU memory. Since absolute PDE residuals naturally increase as the grid becomes finer, the more informative quantity across resolutions is the ratio to the multi-step teacher. This ratio remains well below $1$ at all tested resolutions, indicating that the one-step student consistently preserves its advantage over the teacher as problem size increases. In addition, inference remains practical on a single A100 GPU even at $512 \times 512 \times 2$, requiring 317.00 ms and 3075.57 MB per sample.

\begin{table}[t]
\centering
\small
\caption{Resolution scalability of one-step Phys-Instruct on the Poisson benchmark. ``Ratio to Teacher'' denotes the PDE Error of Phys-Instruct divided by that of the multi-step teacher; lower is better. Inference time and peak memory are measured on a single A100 GPU.}
\label{tab:poisson_resolution_scaling}
\begin{tabular}{cccccc}
\toprule
Resolution & Step & PDE Error & Ratio to Teacher $\downarrow$ & Time (ms) & Peak Mem (MB) \\
\midrule
$32 \times 32 \times 2$   & 1 & 0.02  & 69\% & 16.87  & 357.37  \\
$128 \times 128 \times 2$ & 1 & 0.52  & 41\% & 19.38  & 452.95  \\
$256 \times 256 \times 2$ & 1 & 3.52  & 59\% & 62.59  & 894.08  \\
$512 \times 512 \times 2$ & 1 & 12.33 & 34\% & 317.00 & 3075.57 \\
\bottomrule
\end{tabular}
\end{table}

\subsection{Qualitative Results of Darcy Flow}
We show the generated Darcy flow samples with resolution $32\times32$ under 1-4 steps with Phys-Instruct in Figure~\ref{fig:darcy_uncond}. With sampling steps increasing, PDE error drops.
\begin{figure}
    \centering
    \includegraphics[width=1\linewidth]{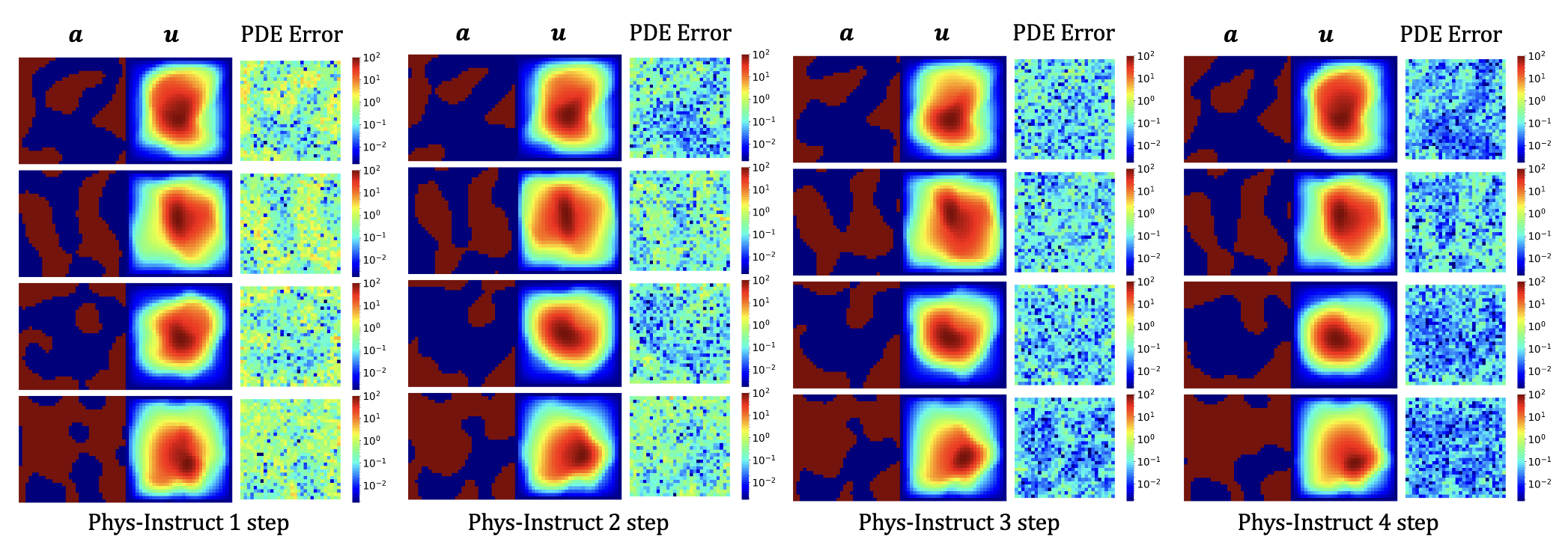}
    \caption{Phys-Instruct generated Darcy flow samples with 1-4 steps (heun sampler).The PDE error column is showing the absolute value of the PDE error field $r(\mathbf{x})_i$ for every point in the grid.}
    \label{fig:darcy_uncond}
\end{figure}

\subsection{Physical Diagnostics Beyond PDE Error}
\label{app:physical_diagnostics}

PDE Error measures consistency with the benchmark-specific discrete equation, but does not alone characterize solution accuracy or other physical properties.
We therefore evaluate complementary diagnostics on Poisson and Burgers' equations and an additional time-dependent Kolmogorov-flow benchmark.
These experiments examine consistency under independent discretizations, independently computed numerical solutions, conservation and balance laws, and spectral structure.

\paragraph{Poisson: Held-Out Discretization and Independent Re-Solving.}
For each unconditionally generated pair $(\mathbf{a},\mathbf{u})$, we independently solve the homogeneous Dirichlet Poisson problem using the generated forcing $\mathbf{a}$.
We compare the generated solution with this numerical solution using the re-solve $L^2$ and $H^1$ errors to assess solution and gradient accuracy.
We also evaluate a held-out fourth-order finite-difference residual (FD4), which is not used for physics guidance during distillation.

Table~\ref{tab:poisson_independent_checks} shows that the one-step student has a lower FD4 residual than EDM-200, but larger re-solve errors.
At 4 steps, Phys-Instruct improves all three diagnostics, indicating that the gain extends beyond the residual discretization used for training to independently evaluated solution values and gradients.

\begin{table}[htbp]
\centering
\small
\caption{Independent physical diagnostics for unconditional Poisson generation. Lower is better for all metrics.}
\label{tab:poisson_independent_checks}
\begin{tabular}{lcccc}
\toprule
Method & Step & FD4 Residual & Re-solve $L^2$ & Re-solve $H^1$ \\
\midrule
\multirow{2}{*}{Phys-Instruct}
& 1 & $2.76{\times}10^{-2}$ & $2.68{\times}10^{-2}$ & $3.66{\times}10^{-2}$ \\
& 4 & $\mathbf{5.76{\times}10^{-3}}$ & $\mathbf{6.74{\times}10^{-3}}$ & $\mathbf{6.49{\times}10^{-3}}$ \\
\midrule
EDM & 200 & $4.23{\times}10^{-2}$ & $1.25{\times}10^{-2}$ & $1.51{\times}10^{-2}$ \\
\bottomrule
\end{tabular}
\end{table}

\paragraph{Burgers': Spectral Residual, Mass Conservation, and Energy Balance.}
For the periodic viscous Burgers' equation, we additionally evaluate an independent spectral residual, a mass-conservation defect, and a viscous energy-balance defect.
The latter two diagnostics test the identities
\begin{equation*}
\frac{\mathrm{d}}{\mathrm{d}\tau}\int_{\Omega_x}u(c,\tau)\,\mathrm{d}c=0,
\qquad
\frac{\mathrm{d}}{\mathrm{d}\tau}\frac12\int_{\Omega_x}u(c,\tau)^2\,\mathrm{d}c
+\nu\int_{\Omega_x}|\partial_c u(c,\tau)|^2\,\mathrm{d}c=0,
\end{equation*}
where $\nu=0.01$ as in Appendix~\ref{apx:pde}.
The energy diagnostic accounts for viscous dissipation rather than assuming energy conservation.

As shown in Table~\ref{tab:burgers_independent_checks}, 4-step Phys-Instruct achieves lower spectral residual, mass defect, and energy-balance defect than EDM-200.
The spectral residual does not decrease monotonically with the student budget, further illustrating that the independent diagnostics capture different properties from the training residual.
We also compare the generated trajectory with an independent numerical rollout initialized from its generated initial frame.
Four-step Phys-Instruct and EDM-200 have similar rollout errors ($0.0300$ and $0.0309$, respectively), while the student improves the conservation and balance diagnostics.

\begin{table}[htbp]
\centering
\small
\caption{Independent physical diagnostics on Burgers' equation. The spectral residual uses an evaluator distinct from the finite-difference residual used for physics guidance. Lower is better for all metrics.}
\label{tab:burgers_independent_checks}
\begin{tabular}{lcccc}
\toprule
Method & Step & Spectral Residual & Mass Defect & Energy-Balance Defect \\
\midrule
\multirow{2}{*}{Phys-Instruct}
& 1 & $\mathbf{1.136}$ & $3.54{\times}10^{-2}$ & $1.968{\times}10^{-1}$ \\
& 4 & $1.194$ & $\mathbf{1.19{\times}10^{-2}}$ & $\mathbf{1.016{\times}10^{-1}}$ \\
\midrule
EDM & 200 & $1.630$ & $1.34{\times}10^{-2}$ & $2.969{\times}10^{-1}$ \\
\bottomrule
\end{tabular}
\end{table}

\paragraph{Kolmogorov Flow: Temporal and Spectral Diagnostics.}
To examine a two-dimensional time-dependent flow, we use the released $\mathrm{Re}=1000$ Kolmogorov-flow dataset of \citet{shu2023physics}\footnote{\url{https://github.com/BaratiLab/Diffusion-based-Fluid-Super-resolution}.}.
The dataset contains periodic vorticity fields on $(0,2\pi)^2$, stored on a $256\times256$ grid with physical-time spacing $\Delta\tau=1/32$.
The governing equation includes advection, viscosity, sinusoidal forcing, and linear damping:
\begin{equation}
\partial_{\tau}\omega+\mathbf{v}\cdot\nabla\omega
=\mathrm{Re}^{-1}\Delta\omega-4\cos(4c_2)-0.1\omega,
\qquad \nabla\cdot\mathbf{v}=0.
\label{eq:kolmogorov_additional}
\end{equation}
Here, $\mathbf{v}$ denotes the velocity recovered from the vorticity using the periodic stream-function relation.
We convert the trajectories into overlapping 3-frame samples
$\mathbf{x}=(\boldsymbol{\omega}_{\tau-\Delta\tau},\boldsymbol{\omega}_{\tau},\boldsymbol{\omega}_{\tau+\Delta\tau})$,
and train an EDM teacher on this representation before distilling an 8-step Phys-Instruct student.
The original Navier--Stokes benchmark uses initial and terminal snapshots; this additional experiment instead evaluates the time-dependent equation locally through three consecutive frames.

For each generated three-frame sample, we evaluate the residual at the middle frame:
\begin{equation}
\begin{aligned}
r_{\mathrm{Kolm}}(\mathbf{x})_{i,j}
:={}&\left[
\frac{\boldsymbol{\omega}_{\tau+\Delta\tau}
-\boldsymbol{\omega}_{\tau-\Delta\tau}}{2\Delta\tau}
+\mathbf{v}_{\tau}\cdot\nabla_h\boldsymbol{\omega}_{\tau}
-\mathrm{Re}^{-1}\Delta_h\boldsymbol{\omega}_{\tau}
\right]_{i,j}\\
&+4\cos(4c_{2,j})
+0.1(\boldsymbol{\omega}_{\tau})_{i,j}.
\end{aligned}
\label{eq:kolmogorov_residual}
\end{equation}
Here, $\nabla_h$ and $\Delta_h$ are Fourier pseudo-spectral
operators, and $c_{2,j}$ denotes the second spatial coordinate.
Following~\eqref{eq:phys_loss}, we compute $\mathcal{R}(\mathbf{x})$
using this residual over all $256^2$ periodic grid points and report
PDE Error as the sample average of $\sqrt{\mathcal{R}(\mathbf{x})}$.
For an independent check, Spectral PDE Error uses the same aggregation
with a residual recomputed using two-thirds dealiasing and standard
zero-mode treatment.
We also report radial-spectrum mismatch relative to EDM-200.
Since EDM-200 defines this spectral reference, its mismatch is zero by construction; this metric measures agreement with the teacher's spectrum rather than with the held-out data spectrum.
Table~\ref{tab:kolmogorov_additional} shows that at the same 8-step budget, Phys-Instruct reduces PDE Error from $150.2$ to $18.86$, Spectral PDE Error from $106.13$ to $17.09$, and spectrum mismatch from $1.880$ to $0.291$.

\begin{table}[htbp]
\centering
\small
\setlength{\tabcolsep}{5pt}
\caption{Unconditional-generation results on $\mathrm{Re}=1000$ Kolmogorov flow using 3 consecutive $256\times256$ vorticity frames. Spectrum mismatch is measured relative to EDM-200. Lower is better for all metrics.}
\label{tab:kolmogorov_additional}
\begin{tabular}{lcccc}
\toprule
Method & Step & PDE Error & Spectral PDE Error & Spectrum Mismatch \\
\midrule
EDM & 8 & $150.2$ & $106.13$ & $1.880$ \\
Phys-Instruct & 8 & $\mathbf{18.86}$ & $\mathbf{17.09}$ & $0.291$ \\
EDM & 200 & $31.46$ & $26.08$ & $0$ (reference) \\
\bottomrule
\end{tabular}
\end{table}

\section{Extra Ablation Studies}
\label{app:sens}
This appendix provides additional ablations that complement the main-text analysis of physics guidance in one-step Phys-Instruct. Unless otherwise stated, all runs follow the \textbf{Default} configuration on Darcy flow, and we report the PDE Error evaluated on one-step samples.
Moreover, we also include ablation study of step operator as described in \ref{subsec:physinstruct_generator}.

\subsection{Sensitivity to the Physics Guidance Weight
\texorpdfstring{$\lambda_{\mathrm{phys}}$}{lambda_phys}}

\begin{table}[t]
\centering
\small
\setlength{\tabcolsep}{4.2pt}
\renewcommand{\arraystretch}{1.05}
\caption{
Sensitivity of training stability to $\lambda_{\mathrm{phys}}$.
All runs use the same setup as \textbf{Default} except for different $\lambda_{\mathrm{phys}}$.
We label runs as \emph{Stable} (training completes normally), \emph{Weak} (training completes but yields limited PDE error improvement), \emph{Unstable} (fails early before the PDE error meaningfully decreases), and \emph{Diverged} (mode collapsed, distillation failed).
}
\label{tab:phys_stability}
\begin{tabular}{c l}
\toprule
$\lambda_{\mathrm{phys}}$ & Status \\
\midrule
$1\mathrm{e}{-4}$ & Unstable (early failure; PDE error remains high) \\
$5\mathrm{e}{-4}$ & Stable, but weak improvement \\
$1\mathrm{e}{-3}$ & Stable \\
$5\mathrm{e}{-3}$ & Stable \\
$1\mathrm{e}{-2}$ & Stable \\
$5\mathrm{e}{-2}$ & Diverged \\
\bottomrule
\end{tabular}
\end{table}

Figure~\ref{fig:app_weight} (a) plots the training trajectories of the PDE Error under different physics guidance weights $\lambda_{\text{phys}}$. Across a broad and practically relevant range, physics guidance consistently accelerates the decay of the PDE error and leads to a lower final plateau. In particular, \emph{moderate} guidance weights yield the most reliable improvements, showing both rapid early reduction and sustained progress throughout training. Figure~\ref{fig:app_weight} (b) further summarizes the final PDE error versus $\lambda_{\text{phys}}$, showing that a moderate range around $5\mathrm{e}{-3}$ achieves the lowest PDE errors.

When $\lambda_{\text{phys}}$ is too small, the physics signal becomes ineffective: training either fails to make meaningful progress or plateaus at a noticeably higher PDE error. Conversely, overly large $\lambda_{\text{phys}}$ can destabilize training and lead to divergence, as summarized in Table~\ref{tab:phys_stability}. These results suggest a \emph{stability--PDE error trade-off}: insufficient weighting under-enforces the PDE constraint, while excessive weighting can dominate the distillation objective and break optimization.

Overall, the default choice of $\lambda_{\text{phys}}$ lies in a stable regime that achieves strong PDE error reduction without sacrificing training stability, and the method remains robust to moderate variations around this value.

\paragraph{Practical Tuning of $\lambda_{phys}$.}
\label{app:tuning}
In practice, selecting an effective $\lambda_{\mathrm{phys}}$ is straightforward. The tuning process is initialized with $\lambda_{\mathrm{phys}} = 10^{-3}$, followed by a short early-training stage that typically requires less than 15 minutes and only a limited number of iterations. The magnitudes of the distribution-matching loss and the physics-residual loss are then compared, with the goal of keeping the physics term within the same order of magnitude as the distribution-matching term so that neither component dominates optimization. If the physics term is substantially smaller, $\lambda_{\mathrm{phys}}$ is increased by a factor of $10$; if it is substantially larger, $\lambda_{\mathrm{phys}}$ is decreased by the same factor. Once the appropriate order of magnitude is identified, only a minor local adjustment is typically needed. Across benchmarks, this empirical procedure is found to be simple and effective, while avoiding an expensive grid search.

\subsection{Sensitivity to the Initial Noise Level $\sigma_{\text{init}}$}

Table~\ref{tab:ablation_sigma} studies the sensitivity of one-step Phys-Instruct to the initial noise level $\sigma_{\text{init}}$, holding all other settings fixed. We find that performance is not highly sensitive to $\sigma_{\text{init}}$ within the tested range: the final PDE Error varies only mildly, indicating that the default $\sigma_{\text{init}}$ is not a fragile tuning choice. This robustness is desirable in practice, as it suggests that one-step distillation with physics guidance does not require aggressive calibration of the initial noise level to achieve good physical consistency.
\label{app:ablation}

\begin{figure}[t]
    \centering
    \begin{subfigure}[h]{0.47\linewidth}
        \centering
        \includegraphics[width=\linewidth]{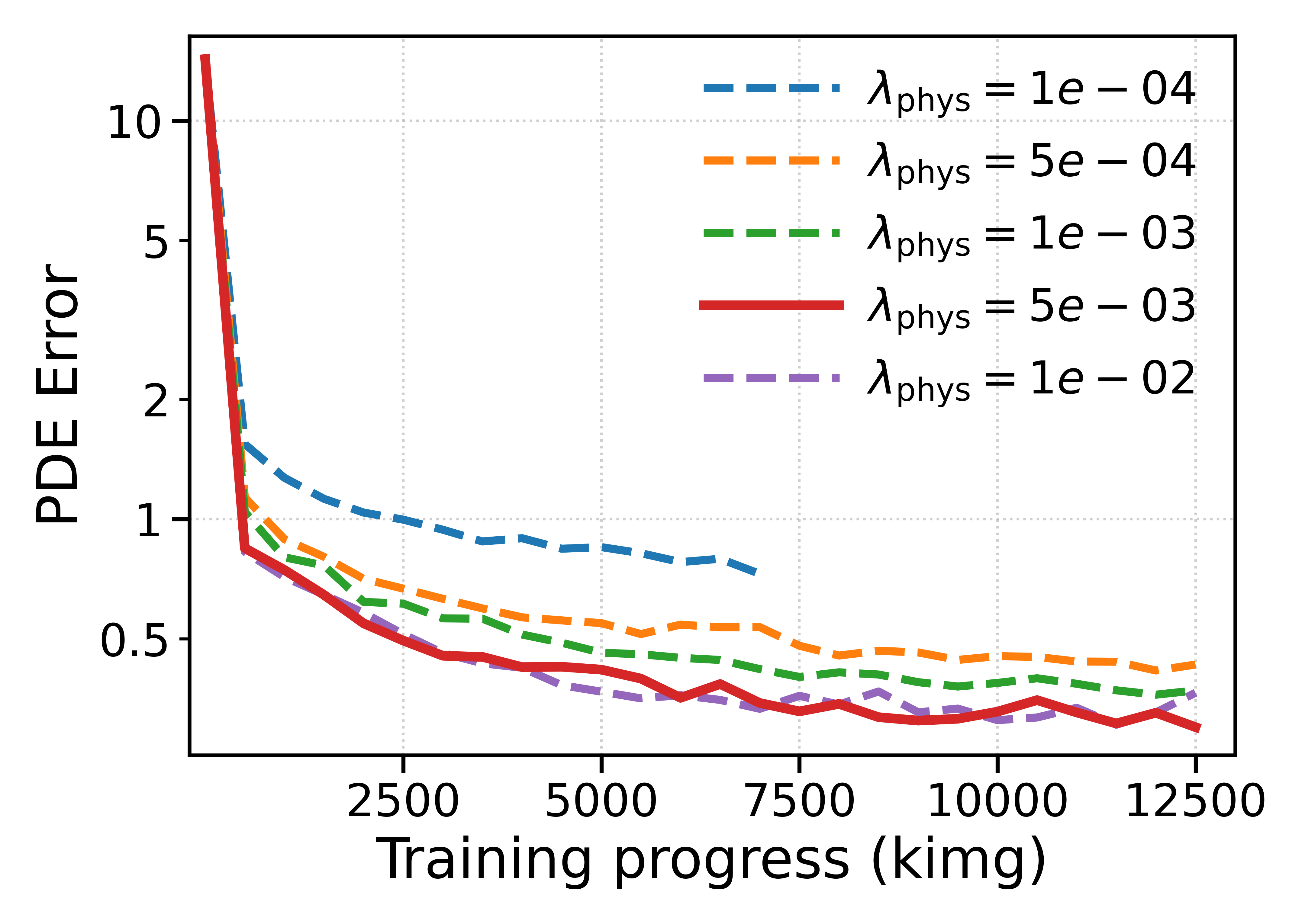}
        \caption{Training dynamics.}
        \label{fig:app_weight_curve}
    \end{subfigure}\hfill
    \begin{subfigure}[h]{0.48\linewidth}
        \centering
        \raisebox{-5mm}{\includegraphics[width=\linewidth]{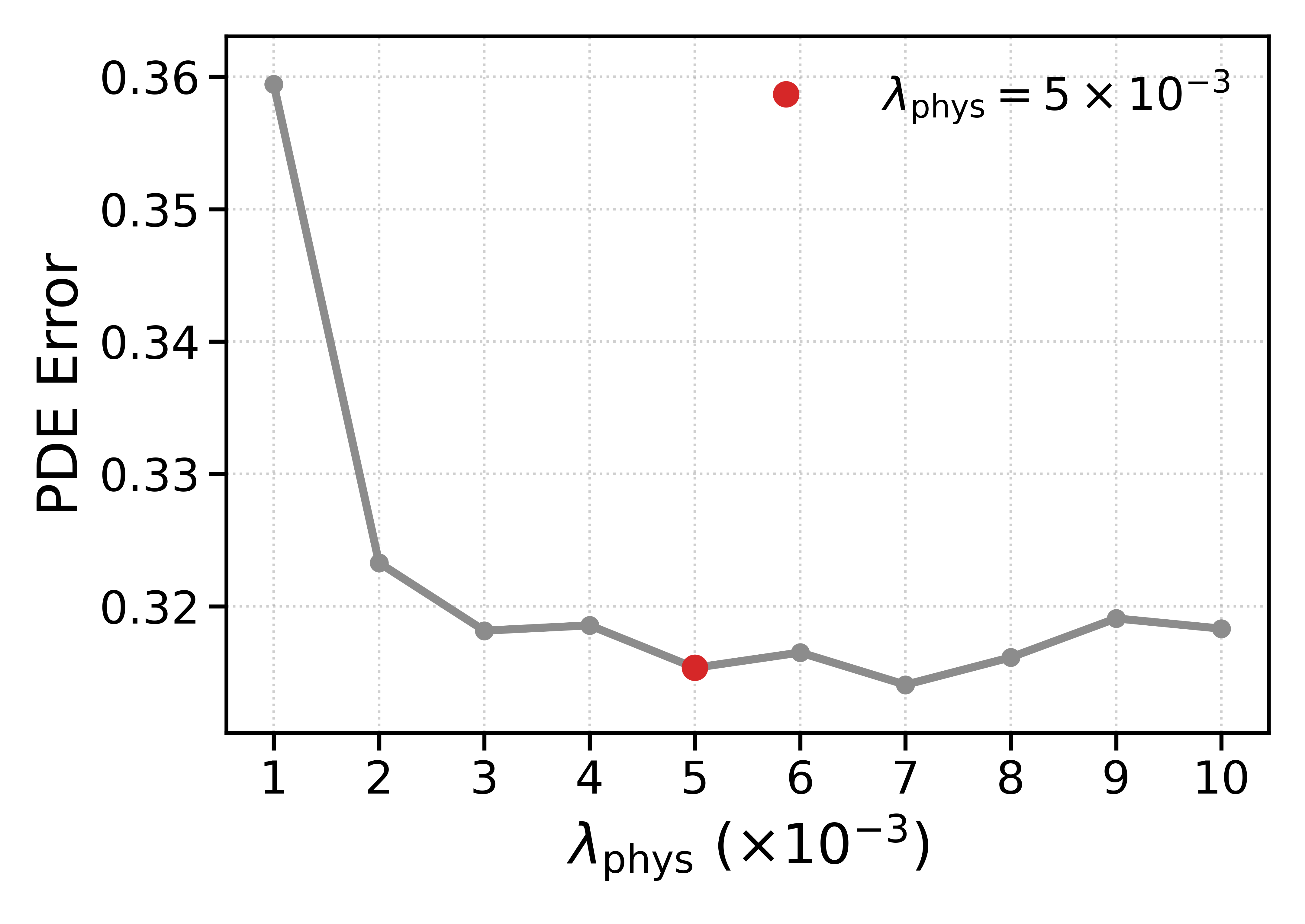}}
        \par\vspace{-2mm}
        \caption{Final PDE error vs.\ $\lambda_{\mathrm{phys}}$.}
        \label{fig:app_weight_summary}
    \end{subfigure}

    \caption{Sensitivity of PDE error to the physics guidance weight $\lambda_{\mathrm{phys}}$ on Darcy flow.
(a) We plot PDE error (log scale) versus training progress for one-step Phys-Instruct with different $\lambda_{\mathrm{phys}}$ (all other settings match \textbf{Default}).
Moderate $\lambda_{\mathrm{phys}}$ yields consistently lower PDE errors, while too small $\lambda_{\mathrm{phys}}$ leads to early instability and a higher error plateau. Diverged runs with overly large $\lambda_{\mathrm{phys}}$ are omitted.
(b) Summary of the final PDE error across $\lambda_{\mathrm{phys}}$ with finer grids between $1\mathrm{e}{-3}$ and $1\mathrm{e}{-2}$. The red dot denotes the PDE error of default $\lambda_{\mathrm{phys}}$. }

    \label{fig:app_weight}
\end{figure}

\begin{table}[t]
\centering
\small
\setlength{\tabcolsep}{4.8pt}
\renewcommand{\arraystretch}{1.08}
\caption{
Sensitivity check of the one-step noise level $\sigma_{\text{init}}$.
All settings use the default configuration except for $\sigma_{\text{init}}$.
}
\label{tab:ablation_sigma}
\begin{tabular}{l c c }
\toprule
$\sigma_{\text{init}}$ & PDE Error ($\times 10^{-1}$)$\downarrow$  \\
\midrule
1.00 &$3.21$    \\
1.50 (default) & $3.15$   \\
3.00 & $3.18$     \\
\bottomrule
\end{tabular}
\end{table}

\subsection{Ablation on Step Operator}
Selecting a step operator is equivalent to specifying the few-step generator’s sampling scheme. We use Heun as the default, and additionally report results with an Euler sampler on Darcy flow. As shown in Table~\ref{tab:sampler}, switching between Heun and Euler changes Phys-Instruct’s performance only marginally, since both employ a fixed, small-step trajectory.

\begin{table}[t]
\centering
\small
\caption{
Unconditional generation results on Darcy flow at resolution $32{\times}32{\times}2$ with two different sampling methods.
We report PDE Error$\downarrow$ and the number of sampling steps per sample (Step).
The lowest PDE Error within each sampler column is \textbf{bolded}.
}
\label{tab:sampler}
\setlength{\tabcolsep}{4.5pt}
\renewcommand{\arraystretch}{0.95}
\begin{tabular}{l c cc}
\toprule
& & \multicolumn{2}{c}{PDE Error ($\downarrow$)} \\
\cmidrule(lr){3-4}
Method & Step & Euler & Heun \\
\midrule

\multirow{4}{*}{Phys-Instruct}
 & 1 & $3.15{\times}10^{-1}$ & $3.15{\times}10^{-1}$ \\
 & 2 & $2.31{\times}10^{-1}$ & $2.18{\times}10^{-1}$ \\
 & 3 & $1.20{\times}10^{-1}$ & $1.35{\times}10^{-1}$ \\
 & 4 & $8.94{\times}10^{-2}$ & $\mathbf{8.49{\times}10^{-2}}$ \\

\bottomrule
\end{tabular}
\end{table}

\section{Conditional Downstream Tasks with ControlNet}
\label{apx:conditional_physinstruct}
The unconditional Phys-Instruct models learn a prior over PDE fields without paired supervision.
In many downstream applications, however, we are given paired data
\begin{equation}
\{(\mathbf{y}_i,\mathbf{x}_i^\star)\}_{i=1}^N,
\end{equation}
where $\mathbf{y}_i$ encodes problem-specific information, such as coefficient fields, boundary conditions,
or observation masks, and $\mathbf{x}_i^\star$ is the corresponding ground-truth PDE field
(solution and, when applicable, coefficients).
Conditional Phys-Instruct leverages this supervision while still using the PDE error as physics guidance.

\paragraph{ControlNet-Style Conditional Generator.}
We build a conditional generator on top of the unconditional Phys-Instruct backbone,
\begin{equation}
\hat{\mathbf{x}} = g_\theta(\mathbf{z},\mathbf{y}), \qquad \mathbf{z}\sim p_z,
\end{equation}
where $\mathbf{y}$ provides task-specific conditioning.
For conditional downstream tasks, we instantiate $g_\theta$ using the one-step Phys-Instruct generator as the backbone,
so the explicit sampling step index $k$ is omitted.
Our backbones are EDM following the SongUNet architecture \citep{Karras2022Elucidating,song2021score}, consisting of an encoder--decoder U-Net with multi-resolution residual blocks
and skip connections.

Following ~\citet{zhang2023adding}, we attach a \emph{control branch} that processes $\mathbf{y}$ and injects
its features into the backbone through zero-initialized residual connections.
Let $\mathbf{e}$ denote the shared timestep (noise-level) embedding produced by the mapping network.
We first map the condition $\mathbf{y}$ to the channel space of the backbone input $\mathbf{z}$ using a convolutional encoder $\psi$,
and form a condition-augmented input for the control branch:
\begin{equation}
\label{eq:control_xc_final}
\mathbf{z}^{c} = \mathbf{z} + \mathrm{ZeroConv}_{\mathrm{in}}\!\big(\psi(\mathbf{y})\big),
\end{equation}
where $\mathrm{ZeroConv}_{\mathrm{in}}$ is a zero-initialized $1\times1$ convolution, hence $\mathbf{z}^{c}\approx \mathbf{z}$
at initialization.
The control encoder mirrors the SongUNet encoder (same resolution hierarchy and block layout) and processes $(\mathbf{z}^{c},\mathbf{e})$
to produce multi-resolution hints $\{\mathbf{h}_\ell(\mathbf{z}^{c},\mathbf{e})\}_\ell$ aligned with the backbone decoder resolutions.
Let $\mathbf{f}_\ell$ denote the backbone feature map immediately before a decoder UNet block at resolution level $\ell$.
We inject the conditioning by projecting the corresponding hint with a zero-initialized $1\times1$ convolution $\mathrm{Proj}_\ell$ and
adding it to the backbone feature:
\begin{equation}
\label{eq:control_inject_final}
\mathbf{f}_\ell' = \mathbf{f}_\ell + \mathrm{Proj}_\ell\!\big(\mathbf{h}_\ell(\mathbf{z}^{c},\mathbf{e})\big),
\qquad
\mathbf{f}_{\ell,\mathrm{out}} = \mathrm{Block}_\ell(\mathbf{f}_\ell',\mathbf{e}).
\end{equation}
Since both $\mathrm{ZeroConv}_{\mathrm{in}}$ and $\mathrm{Proj}_\ell$ are zero-initialized, the injected residuals are initially inactive,
so the conditional generator starts from the pretrained unconditional behavior and gradually learns to incorporate $\mathbf{y}$ through the
residual pathways.

\paragraph{Data Term and Physics Regularization.}
For each pair $(\mathbf{y}, \mathbf{x}^\star)$ and latent noise $\mathbf{z}$, the conditional generator produces a prediction
\[
    \hat{\mathbf{x}} = g_\theta(\mathbf{z}, \mathbf{y}).
\]
Depending on the task, $\mathbf{x}^\star$ may contain full fields or only partially observed components.
We write $\mathcal{O}_\mathbf{y}(\cdot)$ for a generic observation mask that extracts the components of $\mathbf{x}$ that are supervised under condition $\mathbf{y}$.

The data-fidelity term is then
\begin{equation}
    \mathcal{L}_{\mathrm{data}}(\theta)
    =
    \mathbb{E}_{(\mathbf{y}, \mathbf{x}^\star),\, \mathbf{z}}
    \Big[
        \big\|
            \mathcal{O}_\mathbf{y}(\hat{\mathbf{x}})
            -
            \mathcal{O}_\mathbf{y}(\mathbf{x}^\star)
        \big\|_2^2
    \Big].
    \label{eq:cond_data_loss}
\end{equation}

In addition, we evaluate the PDE error on the full prediction $\hat{\mathbf{x}}$ using the same $\mathcal{G}_h$ as in Sec.~\ref{sec 4.2}, leading to the same physics error defined in Eq.~\eqref{eq:physics_loss}.

\paragraph{Conditional Phys-Instruct Objective.}

The Conditional Phys-Instruct generator is trained with a supervised objective regularized by the PDE error:
\begin{equation}
    \mathcal{L}^{\mathrm{cond}}_{G}(\theta)
    =
    \mathcal{L}_{\mathrm{data}}(\theta)
    +
    \lambda_{\mathrm{phys}}\,
    \mathcal{L}_{\mathrm{phys}}(\theta),
    \label{eq:cond_physinstruct_loss}
\end{equation}
where $\lambda_{\mathrm{phys}} > 0$ controls the strength of physics guidance.

This formulation allows a pretrained unconditional Phys-Instruct model to be easily and flexibly adapted to handling different downstream conditional tasks by changing the definition of $\mathbf{y}$ and the observation operator $\mathcal{O}_\mathbf{y}$, while the $\mathcal{L}_{\mathrm{phys}}$ term encourages physically consistent predictions.

\paragraph{Diffusion Baselines and Metrics.}
We compare against DiffusionPDE~\citep{huang2024diffusionpde}, CoCoGen~\citep{jacobsen2025cocogen}, and PIDM~\citep{bastek2025physics}.
All baselines solve conditional tasks via diffusion posterior sampling (DPS), using observation and/or PDE guidance at inference time.
DiffusionPDE follows the original design with DPS guidance, CoCoGen appends an additional ten-step PDE-guided refinement after sampling, and PIDM uses the same DPS-style inference as in DiffusionPDE but replaces the backbone with an unconditional PIDM model.
We report prediction error (relative $L_2$ for continuous fields and absolute error for discrete fields) together with PDE Error.
Since DPS-based solvers typically require long trajectories to converge, we run baseline inference with Step$=2000$ as adopted in \citet{huang2024diffusionpde}.

\paragraph{Tasks.}
We mainly consider three kinds of down stream tasks: (1) Forward problem: inference solution $\mathbf{u}$ from known coefficients $\mathbf{a}$; (2) Inverse problem:  predict coefficient $\mathbf{a}$ given observed $\mathbf{u}$; (3) Reconstruction problem: predict the whole $\mathbf{x}^\star$ given its partial observation.

\section{Extra Downstream Task Results}
\label{app:extra}

\subsection{Comparison with Deterministic Surrogate Models}
We compare Phys-Instruct with deterministic surrogate models on the forward and inverse problem for Darcy flow and Helmholtz on the resolution $128\times128$. The deterministic baselines that we compare are Transolver \citep{wu2024transolver}, FNO \citep{li2021fourier}, PINO \citep{li2024physicsinformed}, DeepONet \citep{lu2021learning}, and PINNs \citep{raissi2019physics}. For FNO, PINO, DeepONet, and PINNs, we report the quantitative values from the corresponding table in \citet{huang2024diffusionpde}. The results are summarized in Table~\ref{tab:darcy_helm_comparison}. While the deterministic
baselines are powerful in these tasks, our method still demonstrates superior performance in all 4 cases, highlighting the strong modeling capabilities of Phys-Instruct.

\begin{table}[t]
\centering
\small
\setlength{\tabcolsep}{10pt}
\caption{Comparison of Phys-Instruct with deterministic baselines on forward and inverse tasks for Darcy Flow and Helmholtz at resolution $128\times128$.}
\begin{tabular}{lccccc}
\toprule
& \multirow{2}{*}{Step}
& \multicolumn{2}{c}{Darcy Flow}
& \multicolumn{2}{c}{Helmholtz} \\
\cmidrule(lr){3-4} \cmidrule(lr){5-6}
&  & Forward & Inverse & Forward & Inverse \\
\midrule
Phys-Instruct (ours) & 1  & \textbf{0.50\%} & \textbf{0.35\%} & \textbf{1.86\%} & \textbf{4.58\%} \\
\midrule
Transolver              & --   & 2.5\% & 7.5\% & 2.3\% & 12.0\% \\
FNO              & --   & 5.3\% & 5.6\% & 11.1\% & 5.0\% \\
PINO             & --   & 4.0\% & 2.1\% & 4.9\% & 4.9\% \\
DeepONet         & --   & 12.3\% & 8.4\% & 17.8\% & 28.1\% \\
PINNs             & --   & 15.4\% & 10.1\% & 18.1\% & 29.2\% \\
\bottomrule
\end{tabular}

\label{tab:darcy_helm_comparison}
\end{table}

\subsection{Additional Forward and Inverse Problem}
We additionally conducted forward problem on Darcy flow and Poisson equation at resolution $32 \times 32$, using the one step Phys-Instruct in Table~\ref{tab:uncond} as the unconditional backbone. We show the results in Table~\ref{tab:darcy_poisson_forward_inverse_32}. One step Phys-Instruct achieves lowest relative/absolute error on the predicted solution, as well as the lowest PDE error.

\begin{table}[t]
\centering
\small
\setlength{\tabcolsep}{3.8pt}
\caption{Forward and inverse results on Darcy flow and Poisson equation at resolution $32\times 32$. We report prediction error and PDE error ($\downarrow$). For forward problems, prediction error is relative error for both Darcy flow and Poisson. For inverse problems, prediction error is absolute error for Darcy flow and relative error for Poisson. Step denotes the number of sampling steps per sample.}
\label{tab:darcy_poisson_forward_inverse_32}
\resizebox{\linewidth}{!}{%
\begin{tabular}{l c cc cc cc cc}
\toprule
& & \multicolumn{4}{c}{Darcy Flow} & \multicolumn{4}{c}{Poisson} \\
\cmidrule(lr){3-6} \cmidrule(lr){7-10}
Method & Step
& \multicolumn{2}{c}{Forward}
& \multicolumn{2}{c}{Inverse}
& \multicolumn{2}{c}{Forward}
& \multicolumn{2}{c}{Inverse} \\
\cmidrule(lr){3-4} \cmidrule(lr){5-6} \cmidrule(lr){7-8} \cmidrule(lr){9-10}
&
& Rel. Err. & PDE Err.
& Abs. Err. & PDE Err.
& Rel. Err. & PDE Err.
& Rel. Err. & PDE Err. \\
\midrule
Phys-Instruct & 1
& \textbf{0.29\%} & \textbf{0.28}
& \textbf{0.49\%} & \textbf{0.22}
& \textbf{2.49\%} & \textbf{0.10}
& \textbf{8.64\%} & \textbf{0.015} \\
DiffusionPDE  & 2000
& 4.61\% & 0.53
& 7.71\% & 1.68
& 7.48\% & 0.22
& 13.34\% & 0.091 \\
CoCoGen       & 2000
& 4.39\% & 0.45
& 7.44\% & 1.59
& 7.25\% & 0.20
& 13.17\% & 0.071 \\
PIDM          & 2000
& 7.34\% & 0.34
& 9.83\% & 1.76
& 5.28\% & 0.12
& 17.48\% & 0.077 \\
\bottomrule
\end{tabular}%
}
\end{table}

\subsection{Reconstruction Problem}
We follow ~\citet{huang2024diffusionpde} and consider a Burgers' equation reconstruction setting under sparse observations. Specifically, we randomly sample 5 out of 128 spatial locations as sensors and provide their measurements over time. Given these partial observations, the goal is to recover the full $128\times128$ spatiotemporal field. Results are summarized in Table~\ref{tab:burgers_partial}. We observe an accuracy--latency trade-off: one-step Phys-Instruct improves over baselines using 200 sampling steps in both prediction error and PDE error, while still trailing diffusion baselines with 2000 steps of roll-outs. Nevertheless, as shown in Figure~\ref{fig:burgers_partial}, the one-step reconstructions preserves the dominant spatiotemporal patterns of the ground truth, supporting the use of one-step sampling as a low-latency reconstruction baseline, which can serve as a strong initialization for subsequent refinement if desired.

\begin{figure}
    \centering
    \includegraphics[width=0.60\linewidth]{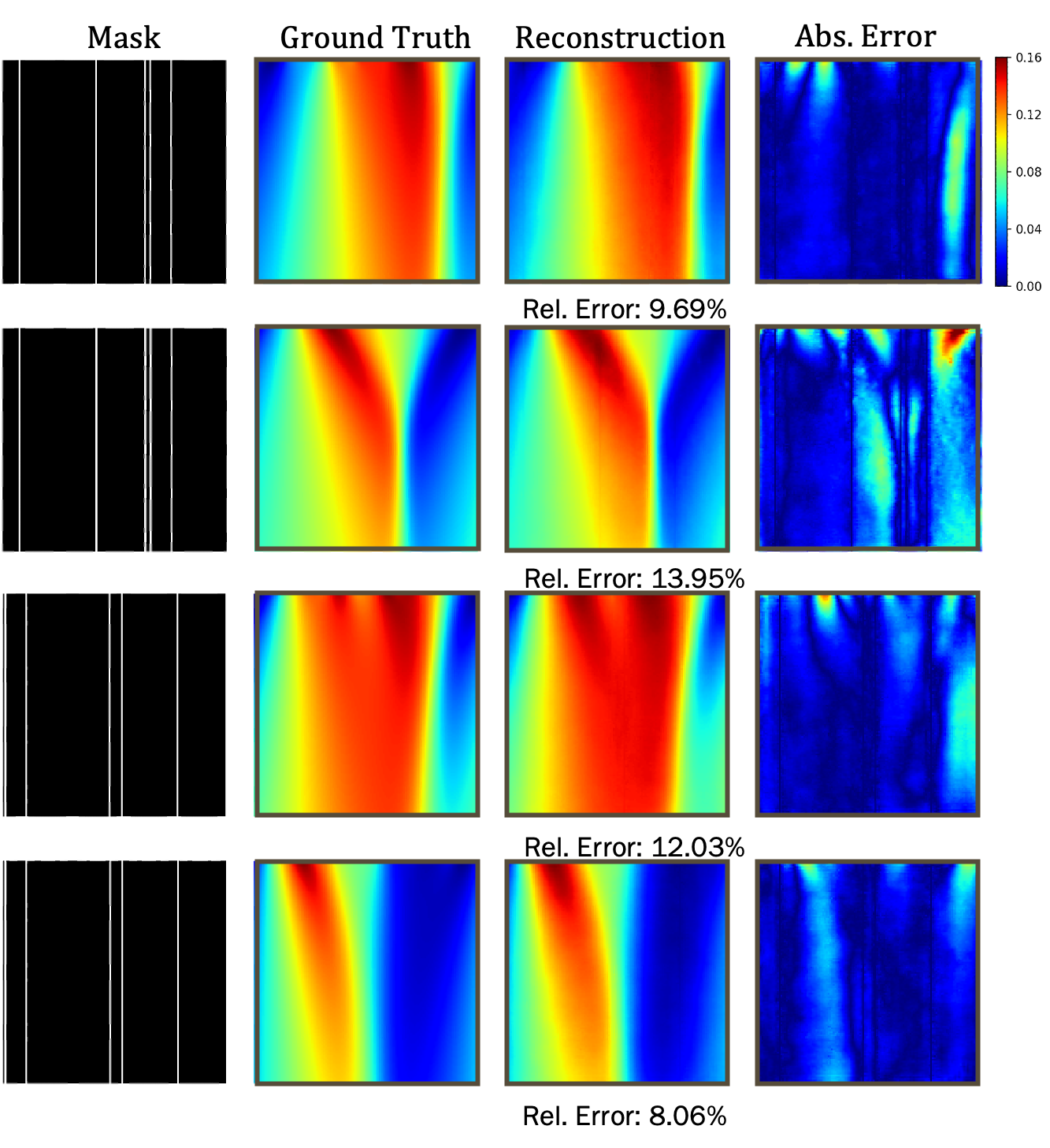}
    \caption{Qualitative result for the reconstruction problem of Burgers' equation.
The mask column shows where the 5 sensors are placed (3.9\% partial observation). The fourth column shows the absolute error $|\mathbf{x}^\star-\hat{\mathbf{x}}|$. }
    \label{fig:burgers_partial}
\end{figure}

\begin{table}[htbp]
\centering
\small
\caption{
Burgers' reconstruction problem.
We report relative errors and PDE Error ($\downarrow$).
Step denotes the number of sampling steps per sample.}
\label{tab:burgers_partial}
\setlength{\tabcolsep}{3.5pt}
\renewcommand{\arraystretch}{0.95}
\begin{tabular}{lccc}
\toprule
& \multicolumn{1}{c}{}
& \multicolumn{2}{c}{Partial Observation} \\
\cmidrule(lr){3-4}
Method
& \multicolumn{1}{c}{Step}
& Rel.\ Error
& PDE Error ($\times 10^{-1}$)
 \\
\midrule
Phys-Instruct & 1
& $12.00\%$
& $7.59$\\
\midrule
DiffusionPDE & 200
& $23.96\%$
& $9.24$
 \\
CoCoGen & 200
& $22.81\%$
& $8.00$
\\
PIDM & 200
& $25.80\%$
& $7.91$
 \\
\midrule
DiffusionPDE & 2000
& $4.66\%$
& $2.42$
 \\
CoCoGen & 2000
& $4.31\%$
& $1.78$
\\
PIDM & 2000
& $6.59\%$
& $1.62$
 \\
\bottomrule
\end{tabular}
\end{table}

\end{document}